\documentclass{article}

\usepackage{microtype}
\usepackage{graphicx}
\usepackage{subfigure}
\usepackage{booktabs} % for professional tables
\usepackage{enumitem}
\usepackage{bbm}
\usepackage{caption}
\usepackage{hyperref}

\usepackage[accepted]{icml2023}

\usepackage{amsmath}
\usepackage{amssymb}
\usepackage{mathtools}
\usepackage{amsthm}

\usepackage[capitalize,noabbrev]{cleveref}

\theoremstyle{plain}
\newtheorem{theorem}{Theorem}[section]
\newtheorem{proposition}[theorem]{Proposition}
\newtheorem{lemma}[theorem]{Lemma}

\theoremstyle{definition}
\newtheorem{definition}[theorem]{Definition}

\theoremstyle{remark}
\newtheorem{remark}[theorem]{Remark}
\usepackage{comment}

\usepackage[textsize=tiny]{todonotes}

\icmltitlerunning{}

\begin{document}

\twocolumn[
\icmltitle{Domain Adaptation via Rebalanced Sub-domain Alignment}




\begin{icmlauthorlist}
\icmlauthor{Yiling Liu}{duke_comb}
\icmlauthor{Juncheng Dong}{duke_ee}
\icmlauthor{Ziyang Jiang}{duke_civil}
\icmlauthor{Ahmed Aloui}{duke_ee}
\icmlauthor{Keyu Li}{duke_ee}
\icmlauthor{Hunter Klein}{duke_ee}
\icmlauthor{Vahid Tarokh}{duke_ee}
\icmlauthor{David Carlson}{duke_civil,duke_biostats}

\end{icmlauthorlist}

\icmlaffiliation{duke_comb}{Program in Computational Biology and Bioinformatics, Duke University School of Medicine, Durham NC, USA}
\icmlaffiliation{duke_civil}{Department of Civil and Environmental Engineering, Duke University, Durham, NC, USA}
\icmlaffiliation{duke_biostats}{Department of Biostatistics and Bioinformatics, Duke University, Durham, NC, USA}
\icmlaffiliation{duke_ee}{Department of Electrical and Computer Engineering, Duke University, Durham, NC, USA}

\icmlcorrespondingauthor{David Carlson}{david.carlson@duke.edu}
%\icmlcorrespondingauthor{Firstname2 Lastname2}{first2.last2@www.uk}

\icmlkeywords{Machine Learning, ICML}

\vskip 0.3in
]



\printAffiliationsAndNotice{} % otherwise use the standard text.

\begin{abstract}
Unsupervised domain adaptation (UDA) is a technique used to transfer knowledge from a labeled source domain to a different but related unlabeled target domain. While many UDA methods have shown success in the past, they often assume that the source and target domains must have identical class label distributions, which can limit their effectiveness in real-world scenarios. To address this limitation, we propose a novel generalization bound that reweights source classification error by aligning source and target sub-domains. We prove that our proposed generalization bound is at least as strong as existing bounds under realistic assumptions, and we empirically show that it is much stronger on real-world data. We then propose an algorithm to minimize this novel generalization bound. We demonstrate by numerical experiments that this approach improves performance in shifted class distribution scenarios compared to state-of-the-art methods.

% We use weights to address class distribution shifting (rebalancing) and prove that our proposed generalization bound is stronger under realistic assumptions. This theory motivates an algorithm that improves performance in shifted class distribution scenarios by minimizing the new generalization bound, which reduces distribution gaps between the source and target domains. We demonstrate that this approach improves performance in such scenarios compared to state-of-the-art methods.

\end{abstract}

\section{Introduction}
Supervised deep learning has achieved unprecedented success in a wide range of real-world applications \cite{ganin2015unsupervised}. However, obtaining labeled data, which is crucial for supervised learning algorithms, may be costly, labor-intensive, and/or time-consuming in certain applications, particularly in medical and biological domains \cite{lu2017deep,li2020transformation}. To address this problem, unsupervised domain adaptation (UDA) has been proposed to make use of available labeled data \cite{farahani2021brief}.  The goal of UDA is to transfer knowledge from a labeled source domain to a different but related unlabeled target domain. Efficient UDA is challenging as models trained on the source domain may not perform well on the target domain due  to discrepancies in the distributions of the two domains hereafter referred to as domain shift \cite{wang2018deep,sankaranarayanan2018learning,deng2019cluster}.

To address the challenge of domain shift, most domain adaptation research has focused on reducing the distributional gap between the source and target domains \cite{shen2018wasserstein,liu2016coupling,isola2017image,tzeng2015simultaneous,tzeng2017adversarial,tzeng2020adapting,ganin2015unsupervised,ganin2016domain,peng2018zero}. These methods are supported by statistical learning theory for transfer learning \cite{ben2006analysis,ben2010theory, mansour2012multiple,redko2017theoretical,li2018extracting}. They also have achieved high performance by learning representations that are both discriminative and domain-invariant \cite{ganin2015unsupervised}. However, these methods are somewhat limited in the sense that they only focus on matching the marginal distribution between domains while ignoring label distributions \cite{deng2019cluster}. This limitation may be significant since class distribution shifts across domains are common phenomena in real-world applications \cite{jiang2020implicit,japkowicz2002class,chawla2009data,tan2019generalized}.

In this paper, we present \emph{Domain Adaptation via Rebalanced Sub-domain Alignment} (DARSA), a novel UDA algorithm that addresses the class distribution shifting. Motivated by theoretical analysis of sub-domain-based UDA methods, DARSA provides state-of-the-art (SOTA) performance in various tasks. Specifically, we attempt to reduce the reweighted classification error and the reweighted discrepancy between the sub-domains of the source task and that of the target task in order to improve classification performance under class distribution shifts. The reweighting is based on a simple intuition: \emph{important sub-domains in the target domain need more attention} (see Section~\ref{sec:weighted_classification_loss} for details). By reweighting the sub-domains, we can significantly alleviate the shifted class distribution problem. This is demonstrated by the superior performance of DARSA compared to those of existing SOTA methods (see Sections~\ref{sec:digit} and ~\ref{sec:TST}). 

The benefits of introducing sub-domain alignment can be visualized in one-dimensional space, as shown in Figure \ref{figure_motiv}(a). After splitting the distributions into corresponding sub-domains (e.g. sub-domains with the same label in the source and target task), we can focus on decreasing the distance between the corresponding source and target sub-domains. We also demonstrate the effectiveness of this method in high-dimensional and complex domain adaptation tasks through an experiment (see Section \ref{sec:digit}) where we transfer information from the MNIST \cite{lecun1998gradient} dataset to the MNIST-M dataset \cite{ganin2016domain} under scenarios with shifted class distribution (e.g. different proportions of labels in the source and target domains). As illustrated in Figure \ref{figure_motiv}(b), comparing with the distance between the source and target domains without aligning sub-domains, our method leads to a smaller distance. 

\begin{figure}[t!]
% \vskip 0.1in
\begin{center}
\subfigure[One-dimensional space]{\includegraphics[width=0.99\columnwidth]{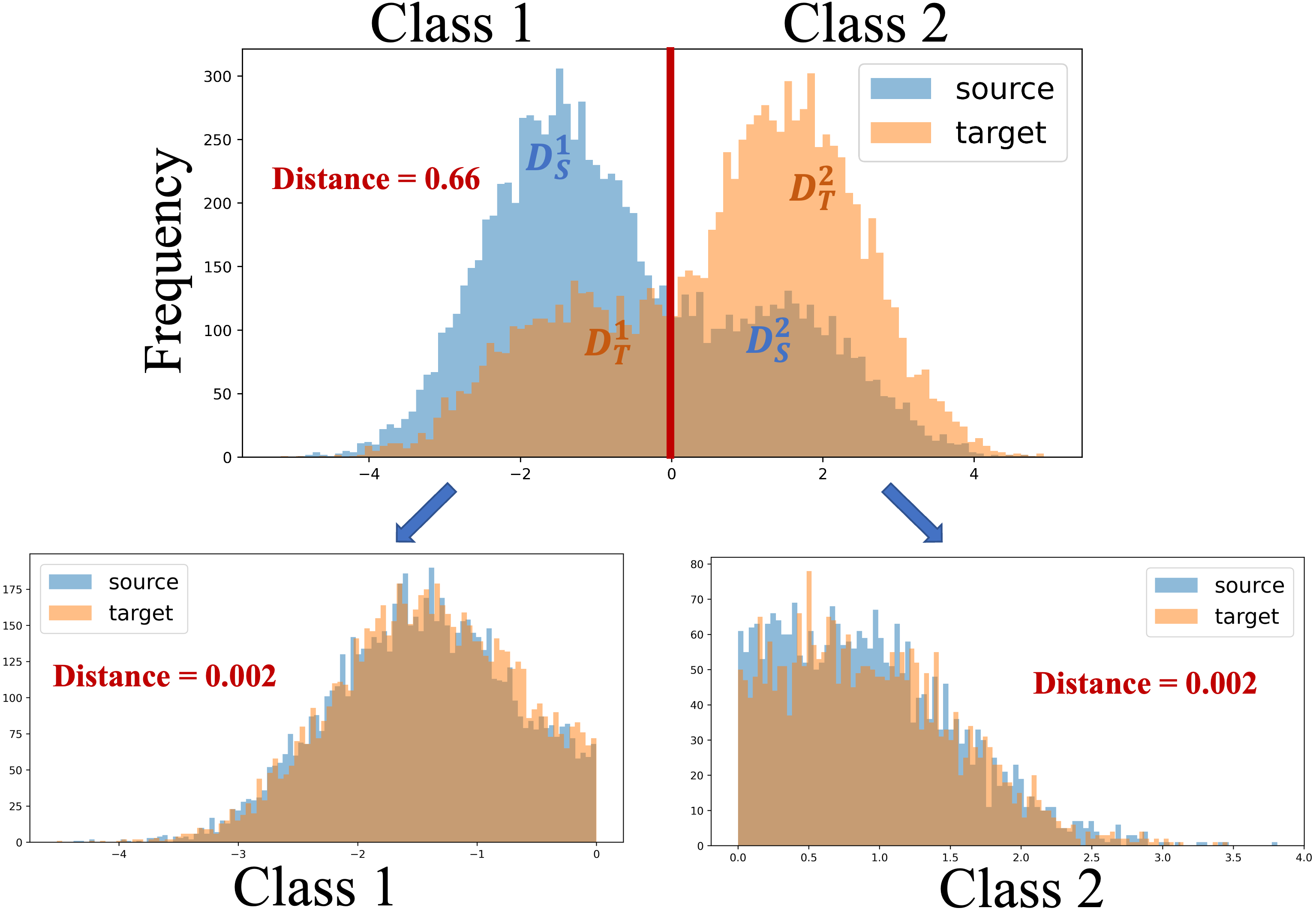}}
\subfigure[High-dimensional space]{\includegraphics[width=0.5\columnwidth]{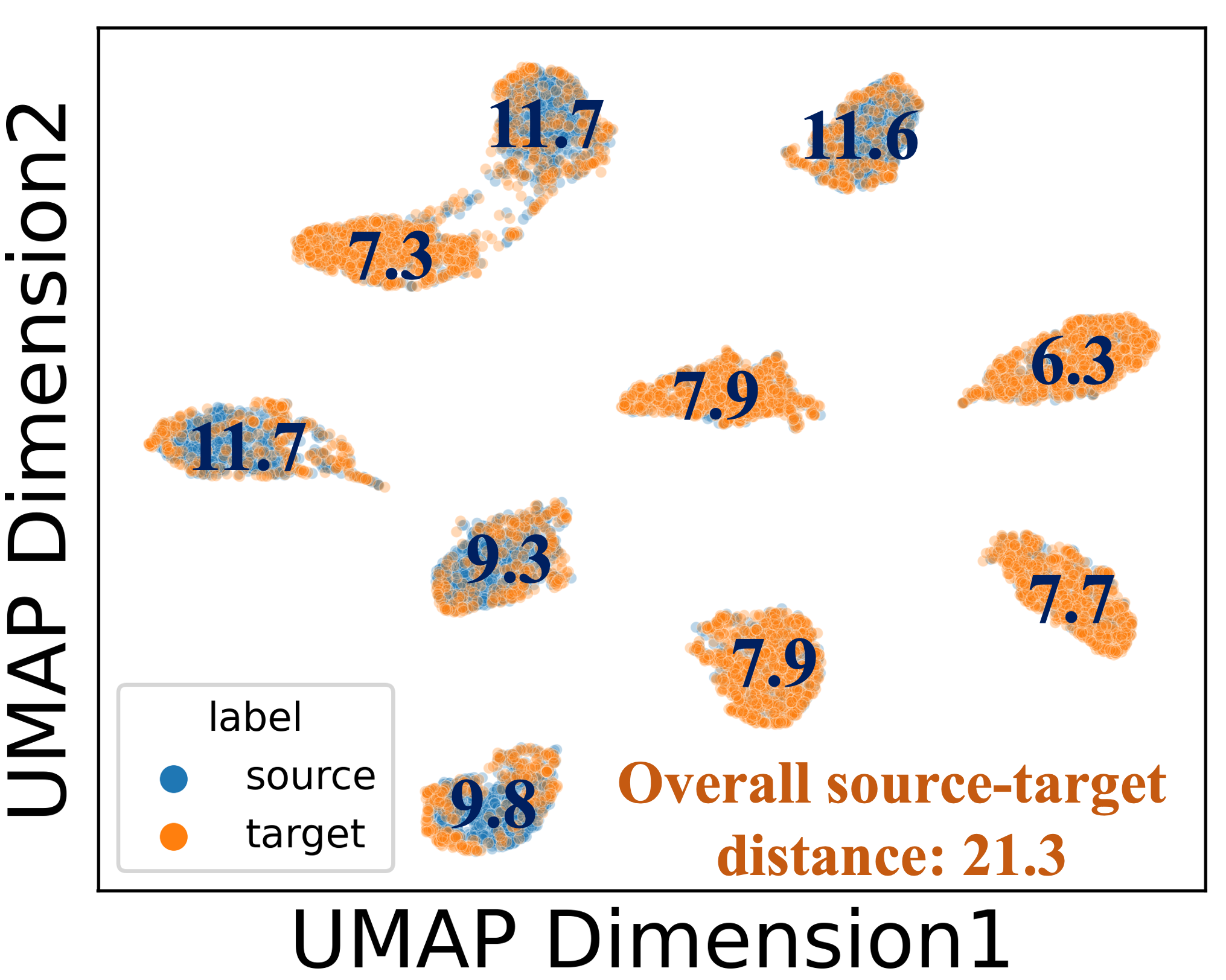}}
\caption{Illustration of the motivation. (a): The source domain (blue), $\mathcal{D_S}$, is sampled from a mixture of two Gaussians centered at $-1.5$ and $1.5$, respectively with weights $0.7$ and $0.3$. The target domain (orange), $\mathcal{D_T}$ is sampled from a mixture of two Gaussian distributions centered at $-1.4$ and $1.6$, respectively with weights $0.3$ and $0.7$. The auxiliary variable Class $\in {1,2}$ divides both domains into $\mathcal{D}^1$ (x $<$ 0) and $\mathcal{D}^2$ (x $\geq$ 0). (b) Distance between corresponding source-target sub-domains is shown for each cluster. Different clusters contain samples of different class labels. The distances are measured using Wasserstein-1 (W1) distance.}
\label{figure_motiv}
\end{center}
\vskip -0.2in
\end{figure}

We summarize our main contributions in this work below:
\begin{itemize}[noitemsep,topsep=0pt]
    \item We theoretically analyze UDA methods that align corresponding source-target sub-domains and establish a new generalization bound.
    \item We prove that our generalization bound is at least as strong as the state-of-the-art bounds for UDA problems with shifted class distribution under realistic assumptions. We also empirically demonstrate that our bound is much stronger on real-world datasets.
    \item Based on the proposed theory, we design an algorithm (DARSA) that reduces rebalanced distribution gaps between the source and target domains and rebalanced classification error by incorporating sub-domain structure.
    \item We demonstrate that DARSA outperforms state-of-the-art methods in shifted class distribution scenarios through experiments on both well-established benchmarks and real-world domain adaptation tasks.
\end{itemize}

\section{Related Work}

% \textbf{Importance Weighting.} Importance weighting is a useful approach for addressing distribution shift between the source and target domains. Methods have been proposed to train a classifier using importance-weighted empirical risk minimization (IW-ERM) in order to correct for label shifting \cite{khetan2017learning,azizzadenesheli2019regularized}. Our approach builds upon the concept of importance weighting by learning a weighted source domain classification error and a weighted source-target domain discrepancy to improve performance, as described in Section \ref{lo}.

\textbf{Discrepancy-based Domain Adaptation.} A common goal of UDA is to reduce the distribution gap between the source and target domains. One approach to achieve this is discrepancy-based methods \cite{tzeng2014deep,long2015learning,sun2016return}, which often use maximum mean discrepancy (MMD) \cite{borgwardt2006integrating} to directly match the marginal distributions of the source and that of the target domains. While MMD is a well-known Reproducing Kernel Hilbert Space (RKHS) metric, it is weaker than the Wasserstein-1 distance \cite{lu2020universal}. Therefore, inspired by WDGRL \cite{shen2018wasserstein}, our work uses the Wasserstein-1 distance as the distance metric. Many discrepancy-based methods also enforce the sharing of the first few layers of the networks between the source and target domains \cite{hassanpour2022survey}. In contrast, our method lifts this restriction by allowing a more flexible feature space. 

\textbf{Adversarial-based Domain Adaptation.} Adversarial-based domain adaptation methods aim to encourage domain similarity through adversarial learning \cite{shen2018wasserstein,liu2016coupling,isola2017image,tzeng2015simultaneous,tzeng2017adversarial,tzeng2020adapting,ganin2015unsupervised,ganin2016domain,peng2018zero,hoffman2018cycada}. These methods are divided into generative methods, which combine discriminative models with a generating process, and non-generative methods, which use a domain confusion loss to learn domain-invariant discriminative features \cite{wang2018deep}. However, many existing algorithms fail to align multi-modal distributions under label shifting scenarios. Additionally, training adversarial networks can be challenging due to mode collapse and oscillations \cite{liang2018generative}.

\textbf{Class-conditional Domain Adaptation.} Class-conditional domain adaptation has been used in a few existing methods to encourage alignment of multi-modal distributions and has shown improved performance in many tasks \cite{deng2019cluster,shi2012information,jiang2020implicit,long2018conditional,snell2017prototypical,pinheiro2018unsupervised}. In contrast to our work, none of these methods provides a theoretical perspective on the benefit of incorporating class-conditional structures. 

%Since target label is not available in classic unsupervised domain adaptation setting, following similar procedure in Cluster Alignment with a Teacher (CAT) \cite{deng2019cluster}, we leverage an implicit ensembling teacher model to obtain pseudo-labels. With the pseudo-labels provided by model predictions, we are able to incorporate the class-conditional structure and provide a better alignment between source and target feature space.

\textbf{Theoretical Analysis of Domain Adaptation.} Many existing domain adaptation methods are inspired by generalization bounds based on the $\mathcal{H}$-divergence \cite{ben2006analysis}. The $\mathcal{H}$-divergence
\cite{ben2006analysis} is a modified version of the total variation distance ($L_1$) that restricts the hypothesis to a given class. These generalization bounds can be estimated by learning a domain classifier with a finite Vapnik–Chervonenkis (VC) dimension. However, this results in a loose bound for most neural networks \cite{li2018extracting}. In our method, we use the Wasserstein distance for two reasons. First, the Wasserstein-1 distance is bounded  above by the total variation distance employed by \citet{ben2010theory}. Additionally, the Wasserstein-1 distance is bounded above by the Kullback-Leibler divergence (a special case of the R\'{e}nyi divergence when $\alpha$ goes to 1 \cite{fournier2015rate}), giving stronger bounds than those presented by \citet{redko2017theoretical,mansour2012multiple}. The second reason for leveraging the Wasserstein distance is that it has stable gradients even when the compared distributions are far from each other \cite{gulrajani2017improved}.

\section{Preliminaries}
Given a labeled source domain $X_S$ with distribution $P_S$ and an unlabeled target domain $X_T$ with distribution $P_T$, our goal is to learn a classifier that can accurately predict labels for the target domain using only the labeled source data and the unlabeled target data. Specifically, we have a labeled source dataset $\{(x_S^i,y_S^i)\}_{i=1}^{N_S}$ and an unlabeled target dataset $\{x_T^i\}_{i=1}^{N_T}$. The source dataset has $N_S$ labeled samples, and the target dataset has $N_T$ unlabeled samples. We assume that the samples $x_S^i \in \mathcal{X} \subseteq \mathbb{R}^d$ and $x_T^i \in \mathcal{X} \subseteq \mathbb{R}^d$ are independently and identically drawn from the probability densities $P_S$ and $P_T$, respectively. For mathematical rigorousness, we further assume that $P_S$ and $P_T$ are probability densities of Borel probability measures in Wasserstein space $\mathcal{P}_1 (\mathbb{R}^d)$, which is the space of probability measures with finite first moment. Our goal is to learn a classifier $f(x)$ that can accurately predict the labels $y_T^i$, given only the labeled source dataset $\{(x_S^i,y_S^i)\}_{i=1}^{N_S}$ and the unlabeled target dataset $\{x_T^i\}_{i=1}^{N_T}$.

\begin{itemize}[topsep=0pt]
    \item \textit{Sub-domain-related notations:} We assume that both $X_S$ and $X_T$ are mixtures of $K$ sub-domains. In other words, we have $P_S=\sum_{k=1}^K w_S^k P_S^k$ and $P_T=\sum_{k=1}^K w_T^k P_T^k$ where we use $P_S^k$ and $P_T^k$ to represent the distribution of the $k$-th subdomain of the source domain and that of the target domain respectively. Note that $\mathbf{w_S}\doteq[w_S^1,\dots,w_S^K]$ and $\mathbf{w_T}\doteq[w_T^1,\dots,w_T^K]$ belong to $\Delta^{K}$ (the $K-1$ probability simplex).
     
    \item \textit{Probabilistic Classifier Discrepancy:} For a distribution $\mathcal{D}$, we define the discrepancy between two functions $f$ and $g$ as:
    \begin{equation*}
    \gamma_\mathcal{D}(f,g) = \mathbb{E}_{x\sim\mathcal{D}}\left[|f(x) - g(x)|\right]
    \end{equation*}
    We use $g_T$ and $g_S$ to represent the true labeling functions of the target and source domains, respectively. We use $\gamma_S(f)\doteq\gamma_{P_S}(f,g_S)$ and $\gamma_T(f)\doteq\gamma_{P_T}(f,g_T)$ to respectively denote the discrepancies of a hypothesis $f$ to the true labeling function for the source and target domains.

\item \textit{Wasserstein Distance:}
The Kantorovich-Rubenstein dual representation of the Wasserstein-1 distance between $P_S$ and $P_T$ can be written as \cite{villani2009optimal}:
\begin{equation*}
W_1(P_S,P_T) = \sup_{||f||_L \leq 1} \mathbb{E}_{P_S}[f(x)] - \mathbb{E}_{P_T}[f(x)]
\end{equation*}
where the supremum is taken over all Lipschitz functions $f$ with Lipschitz constant $L \le 1$ (referred to as the set of $1$-Lipschitz functions). 
\end{itemize}

For notational simplicity, we use $D(X_1,X_2)$ to denote the distance between the distributions of any pair of random variables $X_1$ and $X_2$. For instance, $W_1(\Phi(X_S),\Phi(X_T))$ denotes the Wasserstein-1 distance between the distributions of the random variables $\Phi(X_S)$ and $\Phi(X_T)$ for
any transformation $\Phi$.

\section{Theory} 
%In this section, we present our method, DARSA, for tackling UDA. Given a labeled source domain $X^S$ with distribution $P_S$ and an unlabeled target domain $X_T$ with distribution $P_T$, our goal is to learn a classifier that can accurately predict labels for the target domain, using only the labeled source data and the unlabeled target data. 

%Following are all place-holder. In the sections below, we first state our motivation in Section \ref{sec:31}. Then in Sections XX. In the sections below, we first state our motivation in Section \ref{sec:31}. Then in Sections XX. In the sections below, we first state our motivation in Section \ref{sec:31}. Then in Sections XX.
\label{sec:31}
The importance of the proposed generalization bounds in this section is two-fold: (1) they are theoretically proved to be at least as strong as existing popular upper bounds used in the literature (see Theorem~\ref{theorem:tigher}) and are frequently much stronger on real-world data; (2) they inspire a new framework (see Section~\ref{sec:method}) that achieves state-of-the-art performance on both UDA benchmarks and real-world domain adaptation tasks (see Section~\ref{sec:results}).

%This reweighting is done by assigning weights $w_T^K$ based on the target label distribution. 

% \subsection{Proof of Concept}
% We demonstrate the effectiveness of CCDA-WD in aligning distributions through MNIST $\rightarrow$ MNISTM experiment (more details can be found in Section \ref{sec:41}). As shown in Figure \ref{method_overview_distance}, using CCDA-WD, the distance between the source sub-domains and corresponding target sub-domains is significantly smaller than the distance between the source and target domains. To measure the distances between the distributions, we use the Wasserstein-1 distance. However, calculating the Wasserstein distance can be computationally expensive in practice. In this work, we utilize the Sinkhorn algorithm \cite{cuturi2013sinkhorn} implemented by Feydy et al. \cite{feydy2019interpolating} to approximate the Wasserstein-1 distance. The Sinkhorn algorithm solves an entropy-regularized minimization problem with computational complexity of $O(N^2)$ \cite{liao2022fast}. 

\subsection{Generalization Bounds for Domain Adaptation}
Before presenting our theoretical results about sub-domain-based domain adaptation, we first present an upper bound closely related to the work of \citet{ben2010theory} and \citet{li2018extracting}. It is worth noting that we use the Wasserstein-1 distance in our analysis, as it provides a stronger bound than the total variation distance \cite{redko2017theoretical} employed by \cite{ben2010theory}[Theorem 1].

%Let $g_T$ and $g_S$ be the true labeling function of the target domain and that of the source domain respectively. 
\begin{theorem}[Overall Generalization Bound]
\label{theorem1}
For a hypothesis $f: \mathcal{X} \rightarrow [0,1]$,
\begin{equation}
\gamma_T(f) \leq \gamma_S(f) + (\lambda+\lambda_H) W_1 (P_S, P_T) + \gamma^{\star} 
\end{equation}
where
$\gamma^{\star} = \underset{f \in \mathbb{H}}{\min} \gamma_S(f) + \gamma_T(f)$, $\mathbb{H}$ is a hypothesis class included in the set of $\lambda_H$-Lipschitz functions, and the true functions $g_T$ and $g_S$ are both $\lambda$-Lipschitz functions (as defined in Appendix \ref{def:lipschitz}).
\end{theorem}
\begin{proof}
See in Appendix~\ref{appendix1}.
\end{proof}

\begin{remark}
The upper bound in Theorem~\ref{theorem1} consists of three components: (i) $\gamma_S(f)$ is the performance of the hypothesis on the source domain, (ii) $W_1 (P_S, P_T)$ is the distance between the source and the target domains, and (iii) $\gamma^{\star}$ is a constant that is related to the fundamental difference between the source and the target problems. 
\end{remark}

\begin{remark}
\label{remark:smooth}
For succinctness and clarity of the following analysis, we assume without loss of generality that $\lambda+\lambda_H\leq1$, simplifying the bound to
\begin{equation}
\gamma_T(f) \leq \gamma_S(f) + W_1 (P_S, P_T) + \gamma^{\star} 
\end{equation}

%Suppose that the hypothesis class of $f$ is limited to functions that are $\lambda_H$-smooth and that the true labeling functions are $\lambda$-smooth. There is a constant multiplier ($\lambda + \lambda_H$) on the second term. For simplicity, we assume $\lambda + \lambda_H=1$ and will not write out the constant multiplier in the manuscript.
\end{remark}

Numerous works attempt to solve the transfer learning problem by designing algorithms that attempt to minimize similar generalization bounds (e.g., Theorem 1 of \citet{ben2010theory}).  This approach first requires (1) selecting a mapping $\Phi: \mathcal{X} \rightarrow \mathcal{H}$ to transform the original problem by mapping $X_S$ and $X_T$ into a shared hidden space $\mathcal{H}$, and (2) a hypothesis $h:\mathcal{H}\rightarrow[0,1]$ for prediction. Since $\gamma_T(h\circ\Phi) = \gamma_{\Phi(X_T)}(h)$, with Theorem~\ref{theorem1}, we can have a generalization bound for the performance of the function $h\circ\Phi: \mathcal{X} \rightarrow [0,1]$ on the original target problem: 
\begin{multline}
\gamma_T(h\circ\Phi)=\gamma_{\Phi(X_T)}(h)\\ \leq \gamma_{\Phi(X_S)}(h) + W_1 (\Phi(X_S), \Phi(X_T)) + \gamma^{\star}_{\Phi}.
\end{multline}
If the distance between $\Phi(X_S)$ and $\Phi(X_T)$, i.e., $W_1 (\Phi(X_S), \Phi(X_T))$, is close and the classification error of $h$ on the transformed source problem, i.e., $\gamma_{\Phi(X_S)}(h)$, remains low, then the performance of the hypothesis $h\circ\Phi$ on the \emph{original} target problem can be guaranteed. 
% The inference for the original target problem will be first transformed by $\Phi$ then predicted by $h$ which has a guaranteed performance. 

% \begin{theorem}[Generalization Bound of the Transformed Learning Problem]
% Let $\Phi:\mathcal{X}\rightarrow \mathcal{H}$ be a mapping that transforms the original problems. Then for a hypothesis $h: \mathcal{H} \rightarrow [0,1]$, we have the following regret bound for the classification error of the hypothesis $h\circ\Phi:\mathcal{X}\rightarrow[0,1]$:
% \begin{multline}
%     \gamma_T(h\circ\Phi)=\gamma_{\Phi(X_T)}(h) \\
%     \leq \gamma_{\Phi(X_S)}(f) + W_1 (\Phi(X_S), \Phi(X_T)) + \gamma^{\star}
% \end{multline}
% where $\gamma^{\star}=\min\{\mathbb{E}_{P^\Phi_S}[|g'_S(x) - g'_T(x)|], \mathbb{E}_{P^\Phi_T}[|g'_S(x) - g'_T(x)|]\}$ [TODO: fix definition of $\gamma^*$]. \YL{why we need prime here}
% \end{theorem}
This motivation has led to numerous domain adaptation frameworks that optimize the following objective 
\begin{equation}
    \min_{\substack{\Phi:\mathcal{X}\rightarrow\mathcal{H}\\ h:\mathcal{H}\rightarrow[0,1]}} \gamma_{\Phi(X_S)}(h) + \alpha \; D(\Phi(X_S),\Phi(X_T))
\end{equation}
where $\gamma_{\Phi(X_S)}(h)$ is the classification error of $h$ on the transformed source problem, $D$ is some distance between distributions and $\alpha$ is the balancing weight. In this work, we consider the case where $D$ is Wasserstein-1 distance.

\subsection{Sub-domain-based Generalization Bounds for Domain Adaptation with Shifted Class Distribution}

When taking the sub-domain information into consideration, we can have a stronger bound than the one in Theorem~\ref{theorem1}. We first present several results that will be used to build toward the main theorem. These results themselves may be of interest. First of all, for each subdomain, Theorem~\ref{theorem1} directly leads to the following Proposition:
\begin{proposition}[Sub-domain Generalization Bound]
For $k \in \{1,\ldots,K\}$, where K represents the number of distinct subdomains, for sub-domain $X^k_S$ with distribution $P_S^k$ and $X_T^k$ with distribution $P_T^k$, it holds any $f \in \mathbb{H}$ that  
\label{proposition1}
\begin{equation}
\gamma_{T}^k(f) \leq \gamma_S^k(f) + W_1(P_S^k,P_T^k) + (\gamma^{k})^{\star}
\end{equation}
where $(\gamma^{k})^{\star} = \min_{f \in \mathbb{H}} \gamma_S^k (f) + \gamma_T^k (f)$ is the minimum error can be reached,  $\mathbb{H}$ is a hypothesis class included in the set of $\lambda_H$-Lipschitz functions, the true functions $g_T$ and $g_S$ are both $\lambda$-Lipschitz functions, and $\lambda+\lambda_H\leq1$.
\end{proposition}

Our second result shows that the classification error of any hypothesis $f$ on a domain can be represented by a weighted sum of the classification errors of $f$ on its sub-domains.

\begin{lemma}[Decomposition of the Classification Error] 
For any hypothesis $f\in\mathbb{H}$,
\label{lemma1}
\begin{equation}
\begin{split}
\textstyle\gamma_S(f) =\sum_{k=1}^K w_S^k  \gamma_{S}^k(f), \\
\textstyle\gamma_T(f)  = \sum_{k=1}^K w_T^k  \gamma_{T}^k(f).
\end{split}
\end{equation}
\end{lemma}

With Proposition~\ref{proposition1} and Lemma~\ref{lemma1}, we can have a generalization bound of the target domain with sub-domain information:
\begin{theorem}[Sub-domain-based Generalization Bound]
\label{theorem2}
\begin{multline}
\textstyle    \gamma_T(f,g_T) \leq \sum_{k=1}^K w_T^k \gamma_S^k(f,g_S) \\ 
 \textstyle   + \sum_{k=1}^K w_T^k W_1(P_S^k,P_T^k) + \sum_{k=1}^K w_T^k (\gamma^{k})^{\star}
\end{multline}
\end{theorem}
\begin{proof}
See in Appendix~\ref{theorem2proof}.
\end{proof}

% We will then show that Theorem \ref{theorem2} provides a stronger source-target domain discrepancy term in the generalization bound than Theorem \ref{theorem1} under reasonable assumptions.

We next show that, under reasonable assumptions, the weighted sum of distances between corresponding sub-domains of the source and target domains is at most as large as the distance between the marginal distribution of the source domain and that of the target domain.

\begin{theorem}\label{thm:weight_subdomain_distance_stronger}
If the following assumptions hold,
\begin{itemize}
\item For $k \in \{1,\ldots,K\}$, $P^k_S$ / $P^k_T$ are Gaussian distributions with mean $m_S^k$ / $m_T^k$ and covariance $\Sigma_S^k$ / $\Sigma_T^k$.
\item The distance between the paired source-target sub-domain is smaller or equal to the distance between the non-paired source-target sub-domain, i.e., for $k \neq k^\prime$, we have $W_1(P_S^k,P_T^k) \leq W_1(P_S^k,P_T^{k^\prime})$.
\item There exists an assumed small constant $\epsilon >0$, such that $\underset{1\leq k \leq K}{\max}(\text{trace} (\Sigma_S^k)) \leq \epsilon$ and $\underset{1\leq k \leq K}{\max}(\text{trace} (\Sigma_T^{k^\prime})) \leq \epsilon$. 
% This $\epsilon$ is assumed to be reasonably small.
\end{itemize}
Then 
$$\textstyle \sum_{k=1}^K w_T^k W_1(P_S^k,P_T^k)\leq W_1(P_S,P_T) + \delta_c$$
where $\delta_c$ is $4\sqrt{\epsilon}$. 
% This implies Theorem \ref{theorem2} provides a stronger source-target domain discrepancy term in the generalization bound than Theorem \ref{theorem1}.
\end{theorem}
\begin{proof}
See in Appendix~\ref{dis_stronger_proof}.
\end{proof}

\begin{remark}
$\delta_c$ is a constant dependent only on the variance of the features but not the covariance between features in different dimensions. Moreover, the inequality still holds without $\delta_c$ on the right-hand side as demonstrated by our numerical experiments. To make our bound stronger, we add an intra-clustering loss $\mathcal{L}_{intra}$ in our learning objectives to minimize the $\delta_c$ term (details provided in Section \ref{sec:clustering_loss}).
\end{remark}

\begin{remark}
The assumption of a Gaussian distribution for $X^k$ is not unreasonable since it is often the result of a complex transformation, $\Phi$, and the Central Limit Theorem indicates that the outcome of such a transformation converges to a Gaussian distribution.
\end{remark}

\label{sec:weighted_classification_loss}
In the setting of shifted class distributions, $w_T^k$ and $w_S^k$ can be vastly different. To overcome this, we propose to minimize an objective with the simple intuition that \emph{important sub-domains in the target domain need more attention}. With this motivation, we propose the following objective function for UDA with shifted class distribution:
\begin{equation}
\textstyle    \mathcal{L}(f) = \sum_{k=1}^{K}w^k_T\gamma_S^k(f).
\end{equation}
In other words, $\mathcal{L}$ reweights the sub-domains losses so that the subdomain with more weight in the target domain can be emphasized more. 

We prove that in UDA with shifted class distributions, the sub-domain-based generalization bound is at least as strong as the general upper bound without sub-domain information as demonstrated by the following Theorem. 

\begin{theorem}\label{theorem:tigher}
Let $\mathbb{H}\doteq\{f | f:\mathcal{X}\rightarrow[0,1]\}$ denote a hypothesis space. Under the Assumptions in Theorem~\ref{thm:weight_subdomain_distance_stronger}, if the following assumption hold for all $f\in\mathbb{H}$:
\begin{equation}\label{eqn:ass_uda}
\textstyle \sum_{k=1}^{K}w^k_T\gamma_S^k(f) \leq \sum_{k=1}^{K}w^k_S\gamma_S^k(f), 
\end{equation}
then we have 
\begin{equation*}
  \textstyle   \sum_{k=1}^K w_T^k (\gamma^{k})^{\star} \leq \gamma^{\star}.
\end{equation*}
Further, let 
\begin{multline*}
 \textstyle    \epsilon_c(f) \doteq \sum_{k=1}^K w_T^k \gamma_S^k(f,g_S) \\+ \textstyle \sum_{k=1}^K w_T^k W_1(P_S^k,P_T^k) + \sum_{k=1}^K w_T^k (\gamma^{k})^{\star}
\end{multline*}
denote the sub-domain-based generalization bound, let 
$$
\epsilon_g(f) \doteq \gamma_S(f,g_S) + W_1 (P_S, P_T) + \gamma^{\star}
$$
denote the general generalization bound without any sub-domain information. We have for all $f$,
\begin{equation*}
    \epsilon_c(f) \leq \epsilon_g(f) + \delta_c
\end{equation*}
\end{theorem}

\begin{proof}
See in Appendix~\ref{theorem:tigher_proof}.
\end{proof}

\begin{remark}
We note that Assumption~\eqref{eqn:ass_uda} is likely to hold, as we are minimizing the left hand side of Eq.~\eqref{eqn:ass_uda}.
\end{remark}

\begin{remark}
In addition to theoretically proving that our proposed bound is at least as strong as the existing bound, we empirically observe that the sub-domain-based bound tends to be much stronger (see Figure \ref{figure_motiv} (b) and quantitative results in Section \ref{sec:empircalbound}).
\end{remark}

Inspired by Theorem~\ref{theorem:tigher}, we next propose a framework for imbalanced UDA with the name \textit{Domain Adaptation with Rebalanced Sub-domain Alignment} (DARSA for short).

\section{Methods}\label{sec:method}
Our proposed method, DARSA, focuses on aligning the distributions of the source and target sub-domains (Figure 2(b)) rather than matching the marginal distributions (Figure 2(a)). In DARSA, we divide the source domains into sub-domains based on class labels, and divide target domains into sub-domains using predicted class labels (serving as pseudo labels which have shown success in previous research \cite{lee2013pseudo,deng2019cluster}) for unlabeled target domains. This encourages homogeneous sub-domains (e.g., sub-domains with the same class labels in a classification task) to merge, while separating heterogeneous sub-domains for better decision boundaries.

\begin{figure}[t]
\vskip 0in
\begin{center}
\subfigure[Baseline model]{\includegraphics[width=0.8\columnwidth]{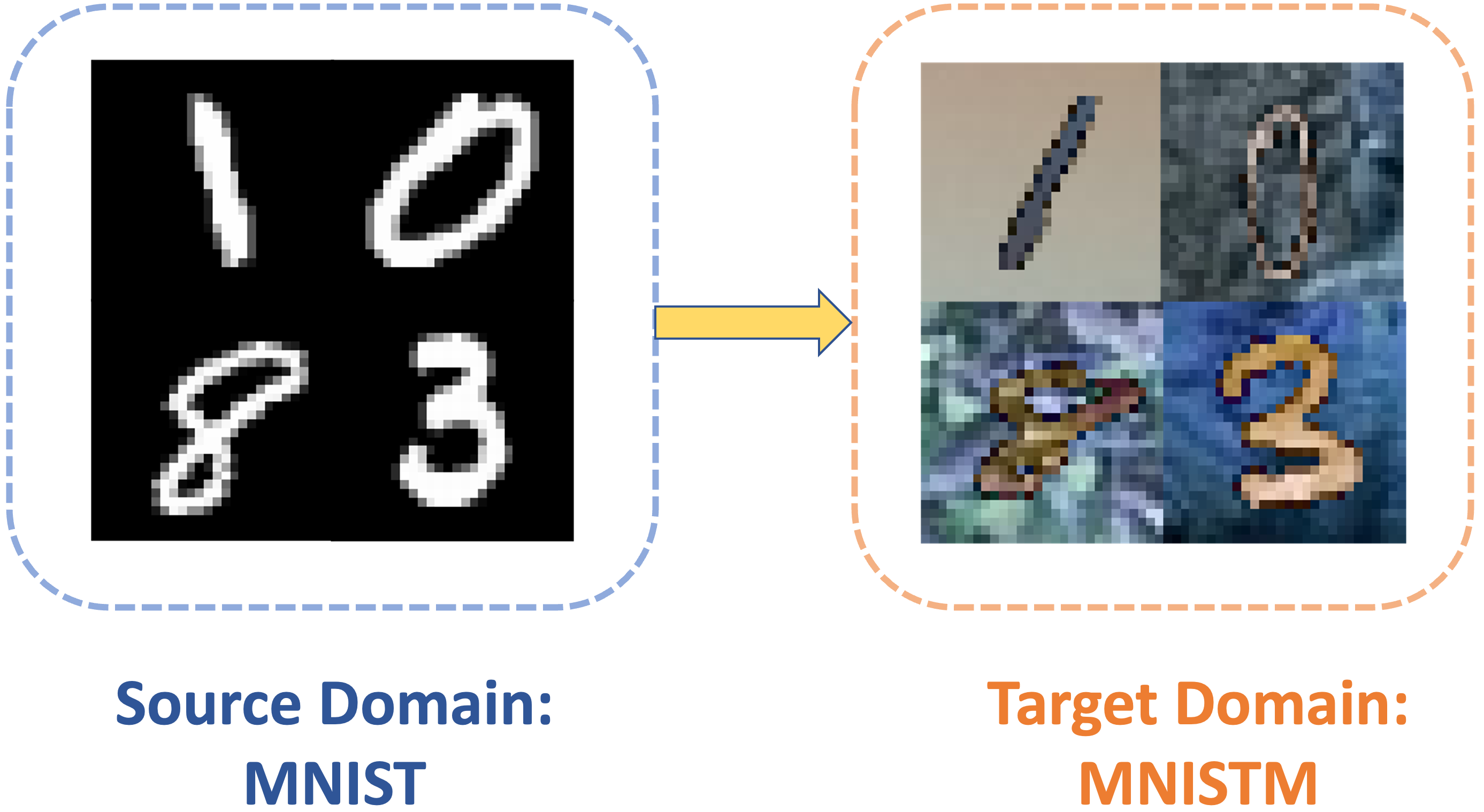}}
\subfigure[Proposed model, DARSA]{\includegraphics[width=0.8\columnwidth]{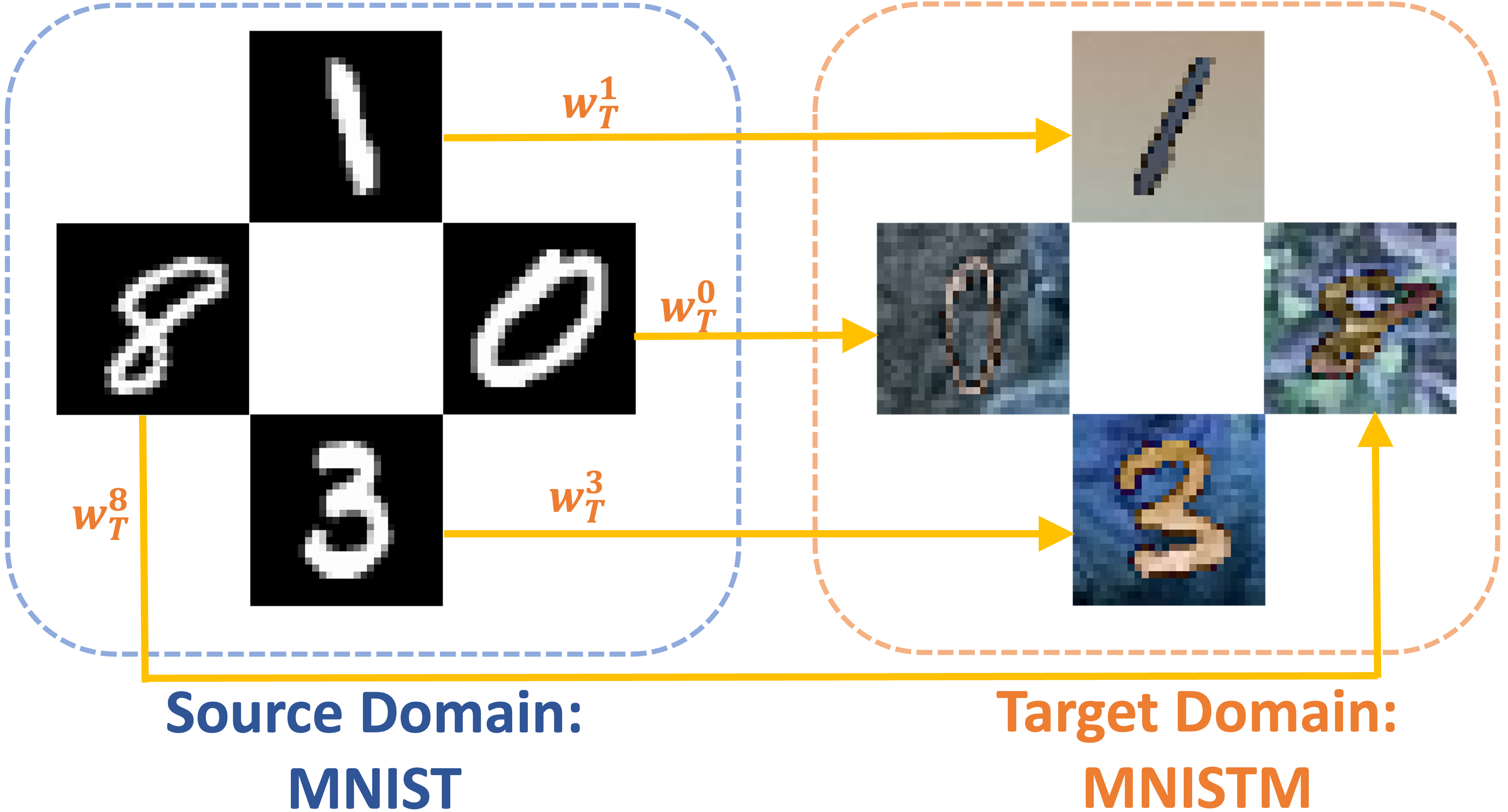}}
\caption{Figure 2(a): classic UDA methods. Figure 2(b): DARSA. $w^k_T$ is the weight (i.e., $\mathbb{P}(y_T=k)$) of the $k$-th subdomain in target domain.}
\label{figure1}
\end{center}
\vskip -0.2in
\end{figure}

Motivated by Theorem \ref{theorem2}, the framework of DARSA, shown in Figure \ref{figure_framework}, is composed of a source encoder $f_E^S$ (parameterized by $\theta_E^S$), a target encoder $f_E^T$ (parameterized by $\theta_E^T$), and a label classifier $f_Y$ (parameterized by $\theta_Y$).  In our implementation, all three functions are defined as neural networks. We first pretrain $f_E^S$ and $f_Y$ with source data, and initialize $f_E^S$ and $f_E^T$ with pretrained weights. The source/target encoder $f_E^S$/$f_E^T$ maps data to feature space, then the classifier $f_Y$ uses extracted features to predict labels. Our framework is outlined in pseudo-code in Appendix \ref{sec:algorithm}. Source labels and predicted target labels are used in the learning objective function for the following purposes:
% \begin{enumerate}[topsep=0pt]
    ($i$) Aligning the sub-domains of the source and the target;
    ($ii$) Estimating weights $\mathbf{w_S}$ and $\mathbf{w_T}$;
    ($iii$) Calculating weighted classification loss and clustering losses.
% \end{enumerate}
We next give a detailed discussion of the learning objective function.

\subsection{Learning Objectives}
\label{lo}
Our learning objectives are motivated by Theorem \ref{theorem2}. Based on Equations \ref{lossy}, \ref{lossd}, \ref{lossc}, and \ref{lossa}, we can represent the learning objectives as follows:
\begin{equation}
\textstyle \min_{\theta_Y,\theta_E^S, \theta_E^T} \lambda_Y \mathcal{L}_Y + \lambda_D \mathcal{L}_D + \mathcal{L}_{C} \\
\end{equation}
where $\mathcal{L}_Y$ represents weighted source domain classification error, $\mathcal{L}_D$ represents weighted source-target domain discrepancy, $\mathcal{L}_{C} = \lambda_c \mathcal{L}_{intra} + \lambda_a \mathcal{L}_{inter}$ represents the clustering loss (details provided in Section \ref{sec:clustering_loss}), and $\lambda_Y$, $\lambda_D$, $\lambda_c$ and $\lambda_a$ are hyperparameters representing weights of each loss. We next elaborate on each one of the losses.

\subsubsection{$\mathcal{L}_Y$(Weighted source domain classification error)}
The weighted source domain classification error in Theorem \ref{theorem2} can be further expressed as:
\begin{equation}
\label{lossy}
\begin{split}
& \textstyle \sum_{k=1}^K w_T^k (\gamma_S^k(f,g_S)) \\
& \textstyle = \sum_{k=1}^K w_T^k\int P_S(x|c=k) |f(x) - g_S(x)|dx \\
& \textstyle = \sum_{k=1}^K w_T^k\int \frac{P_S(c=k|x)P_S(x)}{P_S(c=k)} |f(x) - g_S(x)|dx \\
& \textstyle = \sum_{k=1}^K \frac{w_T^k}{w_S^k}\mathbb{E}_{x\sim D_s} w_S^k(x)|f(x) - g_S(x)| \\
\end{split}
\end{equation}
where variable $c$ represents class, $w_T^k = P_T(c=k),w_S^k = P_S(c=k), w_S^k(x) = P_S(c=k|x)$.  We say that $P_S(c=k|x) = 1$ only when data point $x$ is in class $k$, otherwise $P_S(c=k|x) = 0$. $w_S^k$ can be set to the marginal source label distribution, and $w_T^k$ can be estimated from the target predictions.

From Equation \ref{lossy}, we can express the empirical weighted source domain classification error as:
\begin{equation*}
\begin{split}
\textstyle\mathcal{L}_Y(\theta_Y,\theta_E^S) 
& \textstyle= \frac{1}{N_S} \sum_{x^i \in \mathcal{X}_S} \mathbbm{1}_{y^i = k}\frac{w_T^k}{w_S^k} \ell (\hat{y^i}, y^i)
\end{split}
\end{equation*}
where $\hat{y}^i = f_Y(f_E^S(x^i))$ is the predicted label and $\ell$ can be any non-negative loss function (e.g., cross-entropy loss for classification tasks). %:\mathcal{Y}\times\mathcal{Y}\rightarrow\mathbb{R}_{+}

\begin{figure}[t]
\begin{center}
\centerline{\includegraphics[width=\columnwidth]{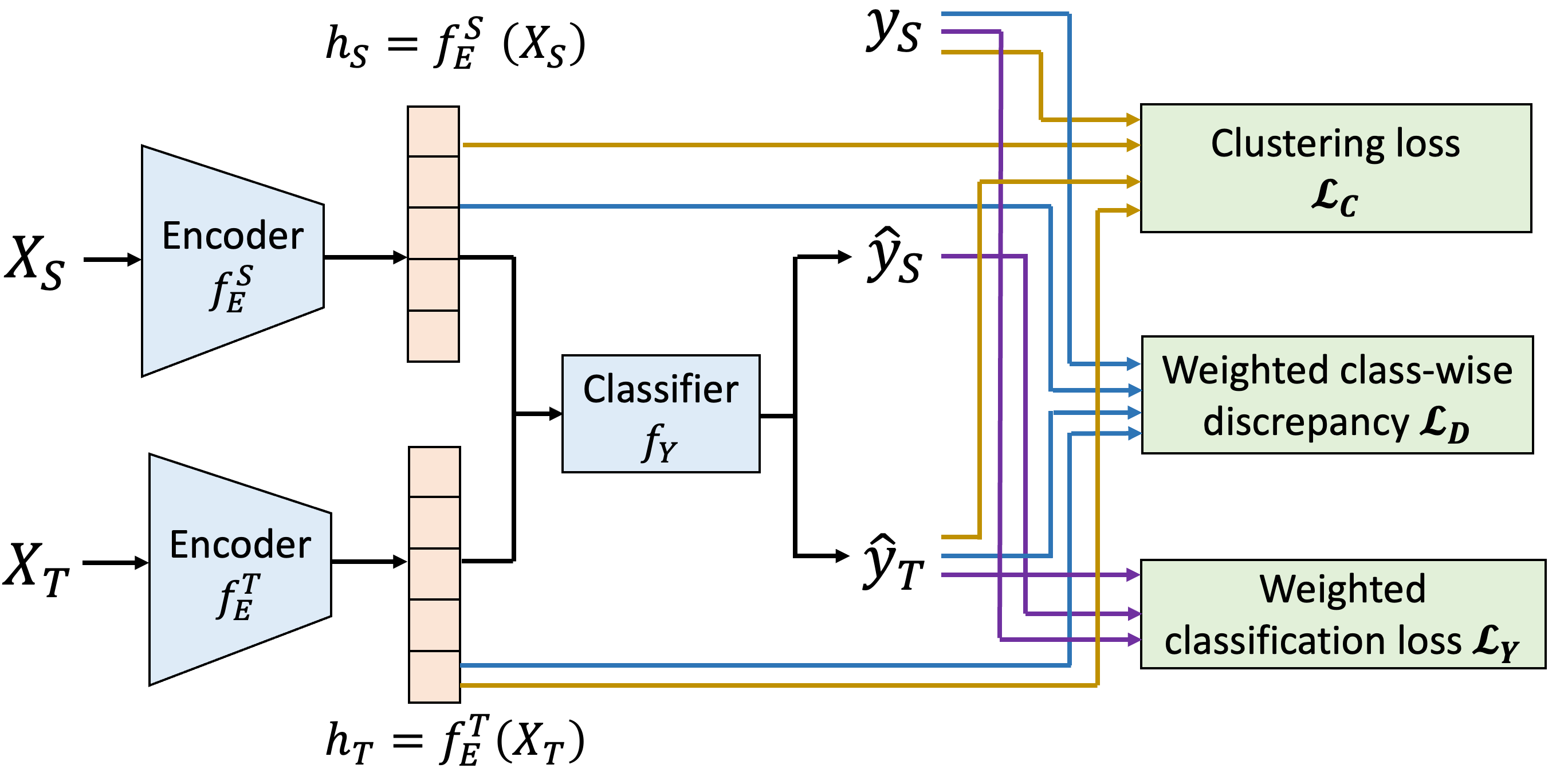}}
\caption{The DARSA framework. Yellow lines representing the clustering loss $\mathcal{L}_C$, blue lines indicating domain discrepancy $\mathcal{L}_D$, and purple lines indicating source classification loss $\mathcal{L}_Y$.}
\label{figure_framework}
\end{center}
\vskip -0.2in
\end{figure}

\subsubsection{$\mathcal{L}_D$ (weighted source-target domain discrepancy)}
The weighted source-target domain discrepancy in Theorem \ref{theorem2} can be further expressed as:
\begin{equation}
\label{lossd}
\begin{split}
\textstyle\mathcal{L}_D(\theta_E^S, \theta_E^T, \theta_Y) & \textstyle= \sum_{k=1}^K w_T^k W_1(P_S^k,P_T^k) \\
& \textstyle= \sum_{k=1}^K w_T^k W_1(f_E^S(x_S^k),f_E^T(x_T^k))
\end{split}
\end{equation}
where $x_S^k$ are source samples with labels $y_S = k$, $x_T^k$ are target samples with predicted labels $\hat{y}_T = k$. In addition, we leverage the Sinkhorn algorithm \cite{cuturi2013sinkhorn} to approximate the Wasserstein metric.

\begin{table*}[t!]
\centering
\caption{Summary of UDA results on the digits datasets with shifted class distribution, measured in terms of prediction accuracy (\%) on the target domain.}
\centering
\begin{tabular}{c|cccc} 
\hline
~  & \begin{tabular}[c]{@{}c@{}}MNIST to \\~MNIST-M\end{tabular}  & \begin{tabular}[c]{@{}c@{}}MNISTM to \\~MNIST\end{tabular} 
 & \begin{tabular}[c]{@{}c@{}}USPS to \\~MNIST\end{tabular}   & \begin{tabular}[c]{@{}c@{}}SVHN to \\~MNIST\end{tabular} \\ 
\hline
DANN \cite{ganin2016domain}   & 63.1       & 93.0  &   59.8  &  64.9 \\
WDGRL \cite{shen2018wasserstein}  &  60.4      & 93.6 &   63.9    &  64.3 \\
DSN \cite{bousmalis2016domain}    & 62.3       & 98.4    &   59.9     &  15.2 \\
ADDA \cite{tzeng2017adversarial}  & 88.2       & 90.7 &  44.8   &  42.4 \\
CAT \cite{deng2019cluster} &    54.1   &  95.4  &   81.0   &  65.8 \\
CDAN \cite{long2018conditional}  &  58.7     &  96.0 &   42.0     &  38.3 \\
pixelDA\cite{bousmalis2017unsupervised} & 95.0       & 96.0 &   72.0     &  68.0  \\
DRANet \cite{lee2021dranet}  & 95.2       & 97.8   &   86.5     &  40.2  \\ 
Source Only  & 47.9      & 91.5  &    40.8    &  53.7 \\ 
\hline
DARSA & \textbf{96.0}      & \textbf{98.8}   &   \textbf{92.6}      &  \textbf{90.1} \\
\hline
\end{tabular}
\label{tab:1}
\vspace{-3mm}
\end{table*}

\subsubsection{$\mathcal{L}_{C}$ (Clustering loss)} %$\mathcal{L}_C$ (
\label{sec:clustering_loss}
Inspired by CAT \cite{deng2019cluster}, the clustering loss $\mathcal{L}_C$ consists of two losses: the \textit{intra-clustering loss $\mathcal{L}_{intra}$} and the \textit{inter-clustering loss $\mathcal{L}_{inter}$}.

% \textbf{Intra-clustering loss ($\mathcal{L}_c$)} is first proposed as SNTG loss \cite{luo2018smooth}. As shown in Figure \ref{figure1}, we aim to direct features from the same label to concentrate and merge, while push features from different labels separate from each other with at least a self-defined distance m. 
% The definition of $\mathcal{L}_c$ is shown as follows:

$\mathcal{L}_{intra}$ was first proposed to direct features of the same label to concentrate as well as to push features of different labels to separate from each other with at least a user-specified distance $m$ \cite{luo2018smooth}. We use the following definition of $\mathcal{L}_c$:
\begin{equation}
\label{lossc}
\begin{aligned}
& \mathcal{L}_{intra}(\theta_E^S, \theta_E^T, \theta_Y)\\ &= \mathcal{L}_{intra} (f_E^S(\mathcal{X_S})) +  \mathcal{L}_{intra} (f_E^T(\mathcal{X_T})) 
\end{aligned}
\end{equation}
with
\begin{align}
% \begin{split}
\begin{aligned}
&  \mathcal{L}_{intra} (f_E^S(\mathcal{X}))\\& =  \frac{1}{N^2} \sum_{i,j=1}^N \Big[\delta_{ij} \|f_E(x_i) - f_E(x_j)\|^2  \nonumber \\
&  + (1-\delta_{ij}) \max \left(0, m - \|f_E(x_i) - f_E(x_j)\|^2\right)\Big]
% \end{split}
\end{aligned}
\end{align}
where $N$ represents the number of samples in the domain $\mathcal{X}$, $\delta_{ij} = 1$ only if $x_i$ and $x_j$ have the same label (using ground truth label in source domain; use predicted label in target domain), otherwise $\delta_{ij} = 0$. $m$ is a pre-defined distance controlling how separate we want each class to be.

$\mathcal{L}_{inter}$ is leveraged to help domain adaptation by aligning centroids of the source sub-domains with that of the corresponding target sub-domains in the representation space. The definition of $\mathcal{L}_{inter}$ is
% \JD{quick thought: maybe we can further improve performance by replacing $\frac{1}{K}$ with $w_T^k$}:

\label{lossa}
\begin{align}
\begin{aligned}
&\textstyle\mathcal{L}_{inter}(\theta_E^S, \theta_E^T, \theta_Y) \\ &= \frac{1}{K} \sum_{k=1}^K \|\text{c(}f_E^S(x_T^k)\text{)} - \text{c(}(f_E^T(x_T^k)\text{)}\|^2
\end{aligned}
\end{align}
where c($\cdot$) calculates the centroids of the sub-domains.

\section{Results}
%\footnotetext[1]{The code to replicate all experiments is available at: \url{https://anonymous.4open.science/r/DARSA/}}
\label{sec:results}
We first demonstrate through empirical evidence that our proposed generalization bound is stronger. After that, we evaluate our proposed method, DARSA, on four benchmark UDA tasks using digit datasets and two tasks using real-world neural datasets. All source-target datasets are resampled to enforce shifted label distributions. We show that our method outperforms all state-of-the-art transfer learning methods under shifted label distributions, including methods specifically designed for computer vision tasks. As shown in Figure \ref{featurespace}, our method learns a feature space that is discriminative and domain-invariant,  resulting in improved performance in target domain predictions.

\subsection{Empirical Analysis of our Proposed Generalization Bound}
\label{sec:empircalbound}
In our empirical analysis of our proposed generalization bound, we evaluate the bound on the MNIST to MNIST-M UDA task (dataset details in Section \ref{sec:digit}). As shown in Figure \ref{figure_emp_compare}, our empirical results demonstrate that our theory, i.e., Theorem \ref{theorem2} provides a stronger generalization bound than Theorem \ref{theorem1}. Additional empirical results to support this claim are provided in Appendix \ref{emp_tighter_more}. 

\begin{figure}[t]
\vskip -0.2in
\begin{center}
\subfigure[Domain discrepancy]{\includegraphics[width=0.48\columnwidth]{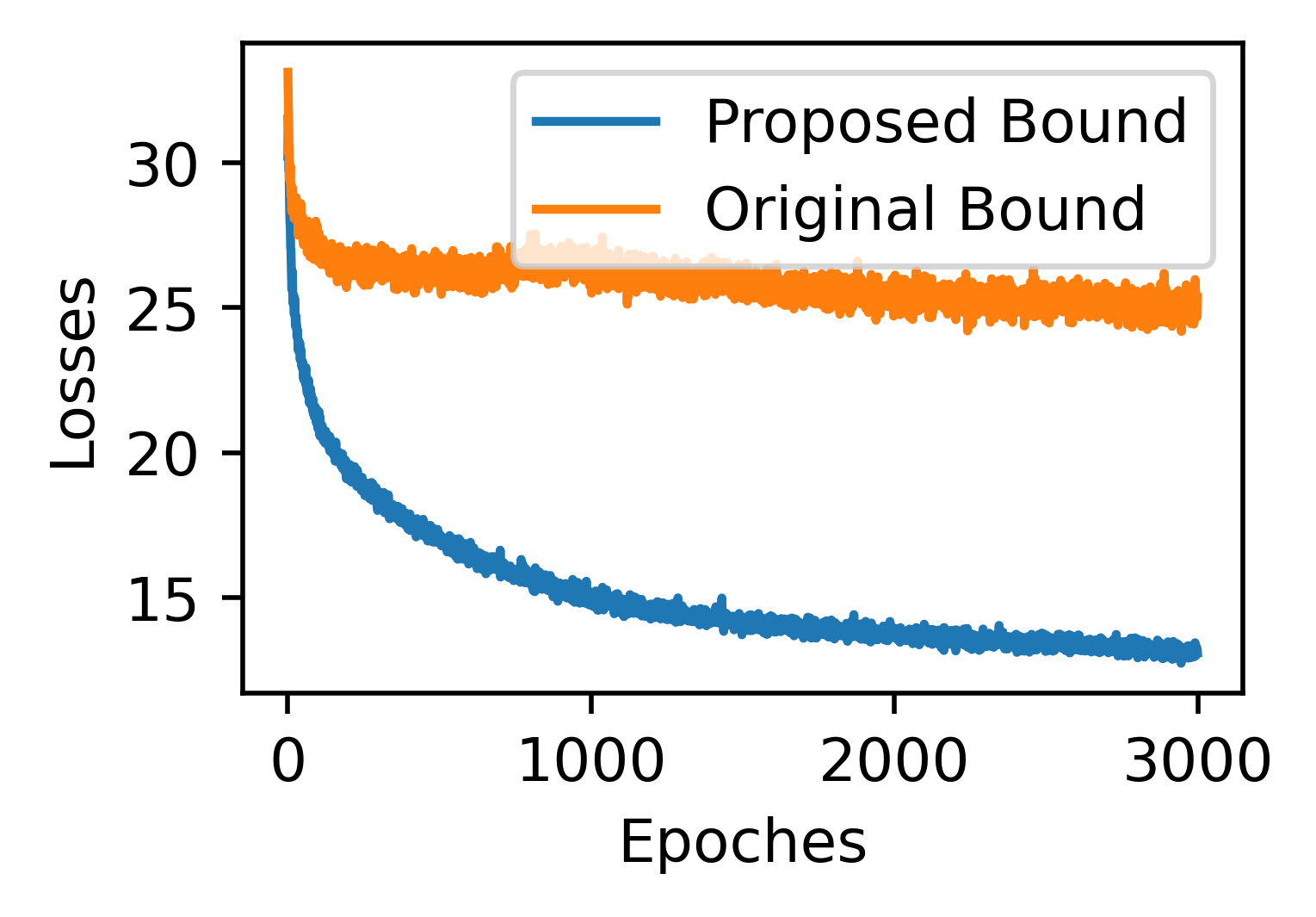}}
\subfigure[Source Classification Loss]{\includegraphics[width=0.51\columnwidth]{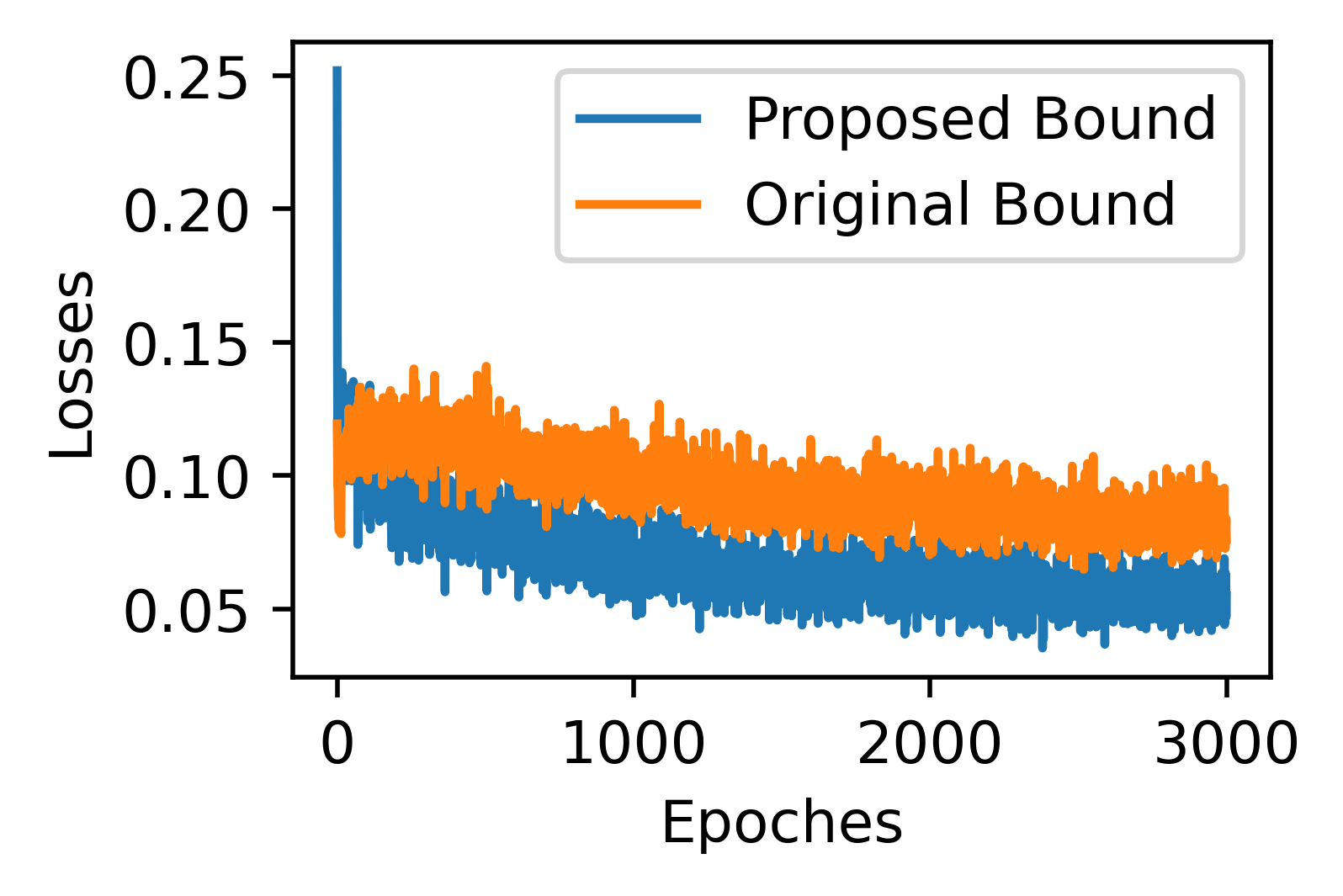}}
\caption{For MNIST to MNIST-M UDA task with shifted class distribution, a) Compare the domain discrepancy term ($\mathcal{L}_D$) in our proposed bound to that in Remark \ref{remark:smooth} 
 b) Compare the source classification term ($\mathcal{L}_Y$) in our proposed bound to that in \ref{theorem1}}.
\label{figure_emp_compare}
\end{center}
\vskip -0.3in
\end{figure}

\subsection{Experiments on Digit Datasets with Shifted Class Distribution}
\label{sec:digit}
We conduct experiments on four digit datasets: MNIST \cite{lecun1998gradient}, MNIST-M \cite{ganin2016domain}, USPS, and Street View House Numbers (SVHN) \cite{netzer2011reading}. To create class distribution shifts, we subsample the datasets so that the proportion of odd digits is three times the proportion of even digits in the source dataset, and vice versa in the target dataset. To ensure a fair comparison, all methods use a subset of the labeled target domain data (around 1000 samples) as a validation set for the hyperparameters search. This validation set performance serves as an upper bound for evaluating the performance of UDA methods. The experimental details (dataset description, model structure and hyperparameters) are provided in appendix \ref{digit_intro}. 

As shown in Table \ref{tab:1}, our model outperforms all competing methods. We note that most of the SOTA comparisons are not specifically designed for shifted class distribution scenarios, and this setting caused issues in several competing methods. We use Adaptive Experimentation (Ax) platform \cite{bakshyopen2019,letham2019constrained}, an automatic tuning approach to select hyperparameters to maximize their performance in domain shifting scenarios (see Appendix \ref{digit_hyper_details}).

\subsection{Experiments on the TST Dataset with Shifted Class Distribution}
\label{sec:TST}
The Tail Suspension Test (TST) dataset \cite{gallagher2017cross} consists of local field potentials (LFPs) recorded from the brains of 26 mice. These mice belong to two genetic backgrounds: a genetic model of bipolar disorder (Clock-$\Delta$19) and wildtype mice. Each mouse is subjected to 3 behavioral assays which are designed to vary stress: home cage (HC), open field (OF), and tail-suspension (TS). We conduct experiments on two transfer learning tasks using these neural activity data: transferring from wildtype mice to the bipolar mouse model and vice versa. We aim to predict for each one second window which of the 3 conditions -  HC, OF, or TS - the mouse is currently experiencing. To create class distribution shifts, we subsample the datasets so that we have 6000 Homecage observations, 3000 Open Field observations, and 6000 Tail Suspension observations in the bipolar genotype dataset and 3000 Homecage observations, 6000 OpenField observations, and 3000 Tail Suspension observations in the wildtype genotype dataset. The experimental details (dataset description, model structure and hyperparameters) are provided in appendix \ref{neural_intro}. As shown in Table \ref{tab:neural}, our model outperforms all competing methods. 

\begin{table}[t]
\centering
\caption{Summary of UDA results on the TST datasets with shifted class distribution, measured in terms of prediction accuracy (\%) on the target domain.}
\centering
\begin{tabular}{c|cc} 
\hline
& \begin{tabular}[c]{@{}c@{}}Bipolar to \\~Wildtype\end{tabular}  & \begin{tabular}[c]{@{}c@{}}Wildtype to \\~Bipolar\end{tabular}\\ 
\hline
DANN \cite{ganin2016domain}   & 79.9       & 81.5   \\
WDGRL \cite{shen2018wasserstein}  &  79.6   & 79.5  \\
DSN \cite{bousmalis2016domain}    & 79.4       & 80.9   \\
ADDA \cite{tzeng2017adversarial}  & 75.1       & 72.6  \\
CAT \cite{deng2019cluster} &    77.3   &  78.6 \\
CDAN \cite{long2018conditional}  &   75.0     &  73.6 \\
Source Only  &  73.8      &  70.4  \\ 
\hline
DARSA & \textbf{86.6}     & \textbf{84.8}  \\
\hline
\end{tabular}
\label{tab:neural}
\vskip -0.3in
\end{table}

\subsection{Analysis of Feature Space}
We visualize the feature spaces learned by DANN \cite{ganin2016domain} and our method, DARSA, using UMAP \cite{sainburg2021parametric}. As shown in Figure \ref{featurespace} (a) and (b), features learned with DARSA form stronger clusters when the labels are the same, and clusters with different labels are more separated from one another. In contrast, DANN \cite{ganin2016domain} fails to learn a good source-target domain alignment in the feature space (shown in Figure \ref{featurespace} (c) and (d)) in the presence of class distribution shifts. This confirms that our method, DARSA, can learn a class-conditional feature space that is discriminative and domain-invariant, which improves performance in target domain prediction. 

%Notably, in Figure \ref{featurespace} (b), the digit 4 and digit 9 clusters are loosely connected. This can be explained by the fact that 4 and 9 look similar in general.

\begin{figure}[ht]
\vskip -0.15in
\begin{center}{\includegraphics[width=\columnwidth]{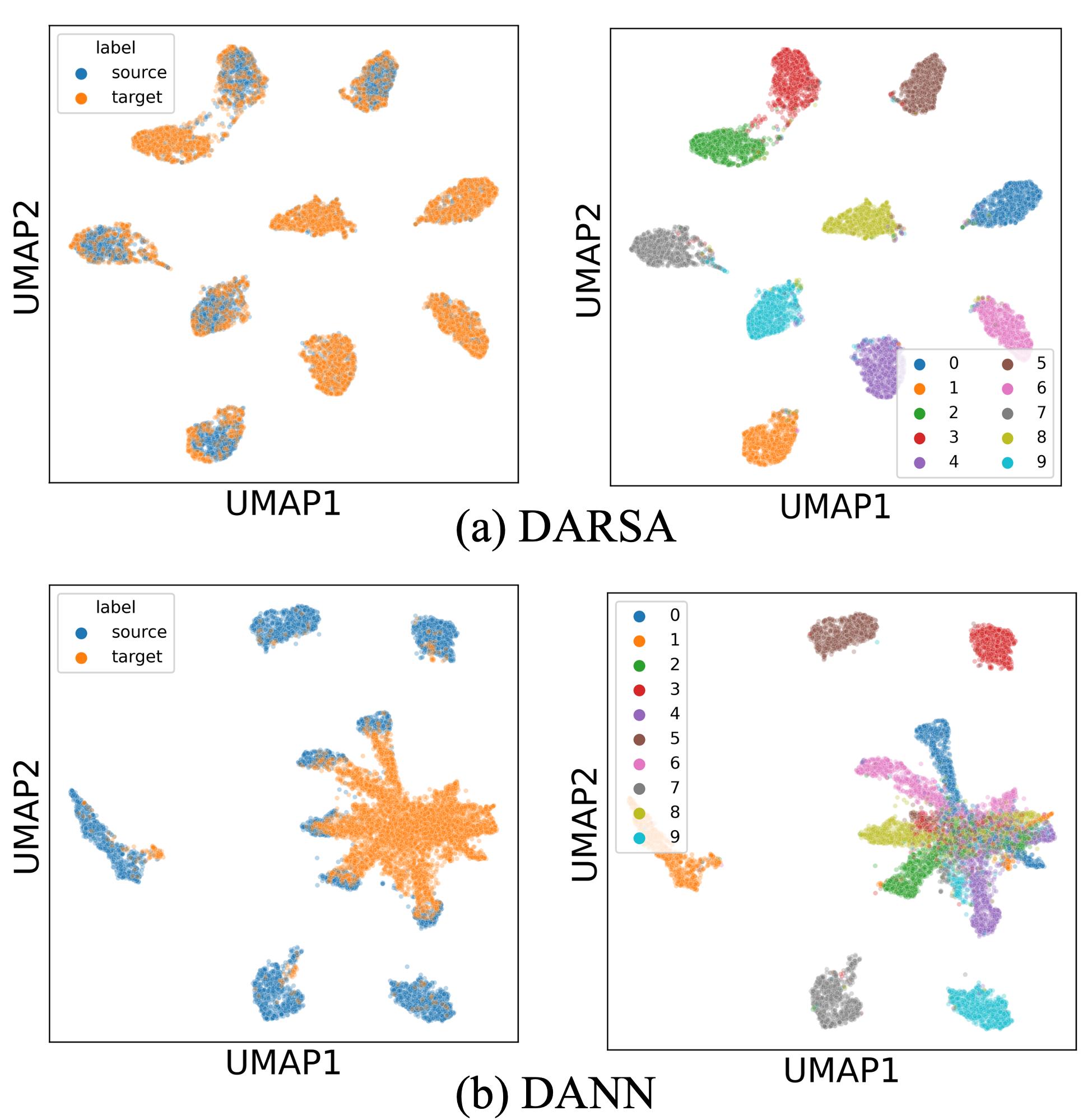}}
\caption{For MNIST to MNIST-M UDA task with shifted class distribution (a) feature space learned by our method, DARSA. (b) feature space learned by DANN. Left panel: colored by source/target; Right panel: colored by true label (digit). The features are projected to 2-D using UMAP.}
\label{featurespace}
\end{center}
\vskip -0.3in
\end{figure}

\section{Discussion and Conclusion}
In this paper, we propose a novel bound for UDA that motivates a novel algorithm with improved performance in various tasks. Since most domain adaptation work is based on reducing the distribution gaps in the source and target domain without incorporating class-conditional structure, we address this challenge to improve alignment between domains during training.  We show that our method outperforms state-of-the-art methods in class imbalance and class distribution shifting scenarios. Our work focuses on classification tasks with source and target having the same number of distinct classes; however, in some cases, this could be addressed by combining similar classes. In addition, by choosing an appropriate clustering algorithm within the framework, this work has the potential to extend to regression tasks. Our current method is also limited to two domains, and future work will extend our method to multiple domains. Lastly, \citet{johansson2016learning} shows that estimating counterfactual outcomes in causal inference under ignorability is mathematically equivalent to UDA between domains $D \in \{0,1\}$ under covariate shift \cite{johansson2020generalization}. With the equivalency, our framework has the potential to be used in causal inference scenarios.

%\clearpage

\bibliography{example_paper}
\bibliographystyle{icml2023}

%%%%%%%%%%%%%%%%%%%%%%%%%%%%%%%%%%%%%%%%%%%%%%%%%%%%%%%%%%%%%%%%%%%%%%%%%%%%%%%
%%%%%%%%%%%%%%%%%%%%%%%%%%%%%%%%%%%%%%%%%%%%%%%%%%%%%%%%%%%%%%%%%%%%%%%%%%%%%%%
% APPENDIX
%%%%%%%%%%%%%%%%%%%%%%%%%%%%%%%%%%%%%%%%%%%%%%%%%%%%%%%%%%%%%%%%%%%%%%%%%%%%%%%
%%%%%%%%%%%%%%%%%%%%%%%%%%%%%%%%%%%%%%%%%%%%%%%%%%%%%%%%%%%%%%%%%%%%%%%%%%%%%%%
\newpage
\appendix
\onecolumn

\section{Definitions and Proofs}
\begin{definition}
\label{def:lipschitz}
    For some $K \geq 0$, the set of $K$-Lipschitz functions denotes the set of functions $f$ that verify: 
    $$
    \|f(x) - f(x')\| \leq K \|x-x'\|, \; \forall x,x' \in \mathcal{X}
    $$
 \end{definition}
 In the coming proofs, we assume that the hypothesis class $\mathbb{H}$ is a subset of $\lambda_{H}$-Lipschitz functions, where $\lambda_{H}$ is a positive constant, and we assume that the true labeling functions are $\lambda$-Lipschitz for some positive real number $\lambda$.
\begin{theorem}[Overall Generalization Bound (Theorem A.8 in \cite{li2018extracting})]
\label{appendix1}
For a hypothesis $f \in \mathbb{H}$ 
\begin{equation}
\gamma_T(f,g_T) \leq \gamma_S(f,g_S) + (\lambda + \lambda_H) W_1(P_S, P_T) + \gamma^{\star} 
\end{equation}
\end{theorem}
where $\gamma^{\star} = \min_{f \in \mathbb{H}} \gamma_S(f) + \gamma_T(f)$ measures how fundamentally different the true labels are for the two domains.

\begin{proof}
Let f be a hypothesis function in $\mathbb{H}$, we have that
\begin{equation}
\gamma_T(f,g_T) = \gamma_T(f,g_T) + \gamma_S(f,g_S) - \gamma_S(f,g_S) + \gamma_S(f,g_T) - \gamma_S(f,g_T)
\end{equation}
And then bound the output term by taking the absolute value of differences:
\begin{equation}
\begin{split}
  \gamma_T(f,g_T) &\leq \gamma_S(f,g_S) + |\gamma_S(f,g_T) - \gamma_S(f,g_S)| + |\gamma_T(f,g_T) - \gamma_S(f,g_T)| \\
  &\leq \gamma_S(f,g_S) + \mathbb{E}_{X_S}[|g_S(x)-g_T(x)|] + |\gamma_T(f,g_T) - \gamma_S(f,g_T)|
\end{split}
  \end{equation}

As stated in \cite{li2018extracting}), the first two terms proceed exactly as by \cite{ben2010theory}; further derivations are not provided. Let $P_S$ and $P_T$ be the densities of $X_S$ and $X_T$, respectively. 
\begin{equation}
\begin{split}
|\gamma_T(f,g_T) - \gamma_S(f,g_T)| & \leq \left |\int (P_T(x)-P_S(x))|f(x) - g_T(x)|dx \right |
\end{split}
  \end{equation}

%Suppose that the hypothesis class of $f$ is limited to functions that are $\lambda_H$-smooth and that the true labeling functions are $\lambda$-smooth. With the additivity properties of Lipschitz smooth functions, we assume $h(x)= f(x) - g_T(x)$ is at least $\lambda + \lambda_H$ smooth. In addition, the absolute value of $h (x)$ is also in the range of [0,1]. Therefore, the previous equation can be bounded by:
Since our hypothesis class $\mathbb{H}$ is assumed to be $\lambda_H$-Lipschitz and the true labeling functions are $\lambda$-Lipschitz, we have that for every function $f\in \mathbb{H}$, $h:x\mapsto |f(x) - g_{T}(x)|$ is $\lambda+\lambda_{H}$-Lipschitz and it takes its values in $[0,1]$.
Therefore,
\begin{equation}
\begin{split}
|\gamma_T(f,g_T) - \gamma_S(f,g_T)| & \leq \left |\sup_{h: \mathcal{X} \rightarrow [0,1], ||h||\leq \lambda + \lambda_H} \int (P_T(x) - P_S(x))h(x)dx \right| \\
  & = \left|\sup_{h: \mathcal{X} \rightarrow [0,1], ||h||\leq \lambda + \lambda_H}  (\mathbb{E}_{X_T}[h(x)] - \mathbb{E}_{X_S}[h(x)]) \right| \\
\end{split}
  \end{equation}

Note that due to the symmetric nature of the function space (i.e if $h$ is K-Lipschitz then $-h$ is K-Lipschitz) we can just pick either side to lead with and drop the absolute value, yielding 

\begin{equation}
\begin{split}
  |\gamma_T(f,g_T) - \gamma_S(f,g_T)| 
  & \leq\left|\max_{h: \mathcal{X} \rightarrow [0,1], ||h||\leq \lambda + \lambda_H}  (\mathbb{E}_{X_T}[h(x)] - \mathbb{E}_{X_S}[h(x)]) \right| \leq (\lambda + \lambda_H) W_1(P_S,P_T)
  \end{split}
 \end{equation}
  
Following the Theorem 2 of \citet{ben2010theory}, we can also easily bound the target error $\gamma_T(f,g_T)$ by:
\begin{equation}
\gamma_T(f,g_T) \leq \gamma_S(f,g_S) + (\lambda + \lambda_H) W_1 (P_S,P_T) + \gamma^{\star}
 \end{equation}
where $\gamma^{\star} = \min_{f \in \mathbb{H}} \gamma_S (f,g_S) + \gamma_T (f,g_T)$ is the minimum error can be reached.
\end{proof}

\begin{lemma}[Decomposition of the Classification Error] For any hypothesis $f\in\mathbb{H}$,
\label{appendix2}
\begin{equation}
\begin{split}
\gamma_S(f) =\sum_{k=1}^K w_S^k  \gamma_{S}^k(f), \\
\gamma_T(f)  = \sum_{k=1}^K w_T^k  \gamma_{T}^k(f).
\end{split}
\end{equation}
\end{lemma}

\begin{proof}
\noindent We can write out $\gamma_S(f,g_S)$ as clustering specific component. Here we use c to represent the clustering index.
\begin{equation}
\begin{split}
  \gamma_S(f,g_S) & = \mathbb{E}_{X_S}\left[|f(x) - g_S(x)|\right] \\
  & = \int P_S (x) |f(x) - g_S(x)|dx \\
  & = \int \sum_{k=1}^K w_S^k P_S(x|c=k) |f(x) - g_S(x)|dx \\
   & = \sum_{k=1}^K w_S^k \int P_S(x|c=k) |f(x) - g_S(x)|dx \\
   & = \sum_{k=1}^K w_S^k  \int P_S(x|c=k) |f(x) - g_S(x)|dx \\
   & = \sum_{k=1}^K w_S^k  \gamma_{S}^k(f,g_S)
\end{split}
\end{equation}
\end{proof}

With similar proof, we have:
\begin{equation}
\gamma_T(f,g_T) =\sum_{k=1}^K w_T^k  \gamma_{T}^k(f,g_T) 
\end{equation}

\begin{theorem}[Class-conditional Generalization Bound]
\label{theorem2proof}
\begin{equation}
    \gamma_T(f,g_T) \leq \sum_{k=1}^K w_T^k \gamma_S^k(f,g_S) 
    + \sum_{k=1}^K w_T^k W_1(P_S^k,P_T^k) + \sum_{k=1}^K w_T^k (\gamma^{k})^{\star}
\end{equation}
\end{theorem}

\begin{proof}
\begin{equation}
\begin{aligned}
\gamma_T(f,g_T) & \overset{\mathrm{Lemma \ref{lemma1}}}{=} \sum_{k=1}^K w_T^k  \gamma_{T}^k(f,g_T) \\
&  \overset{\mathrm{Proposition \ref{proposition1}}}{\leq}   \sum_{k=1}^K w_T^k \{\gamma_S^k(f,g_S) + W_1 (P_S^k,P_T^k) + (\gamma^{k})^{\star}\}
\end{aligned}
\end{equation}
\end{proof}

% \begin{definition}[2-Wasserstein distance （\cite{delon2020wasserstein}）]
% \label{defwass}
% If $P_S$ and $P_T$ are two probability measure in $\mathcal{P}(\mathbb{R}^d)$. Define $\Pi(P_S,P_T) \subset \mathcal{P}(\mathbb{R}^d \times \mathbb{R}^d)$ as subset of joint probability distributions $\gamma$ on $(\mathbb{R}^d \times \mathbb{R}^d$ with marginal distributions $P_S$ and $P_T$. Then the 2-Wasserstein distance $\mathbb{W}$ between $P_S$ and $P_T$  is defined as:
% \begin{equation}
%   W_2^2(\mathcal{P}_S,\mathcal{P}_T) = \inf_{\gamma \in \Pi(P_S,P_T)} \int_{\mathbb{R}^d \times \mathbb{R}^d} ||x_S - x_T||^2 d\gamma(x_S,x_T)
% \end{equation}
% \end{definition}

% \begin{definition}[Wasserstein-like distance between Gaussian Mixture Models (Definition 2 in \cite{delon2020wasserstein})]
% \label{defwassm}
% $P_S = \sum_{k} w_S^k P_S^k$ and  $\mathcal{P_T} = \sum_{k} w_T^k P_T^k$ be two Gaussian mixtures. We define:
% \begin{equation}
%    MW_2^2(P_S,P_T) = \inf_{\gamma \in \Pi(P_S,P_T) \cap GMM_{2d}(\infty)} \int_{\mathbb{R}^d \times \mathbb{R}^d} ||x_S - x_T||^2 d\gamma(x_S,x_T)
%  \end{equation}
%  where $GMM_{2d}(\infty)$ represents the set of all finite Gaussian mixture distributions.
% \end{definition}

% \begin{lemma}[Lemma 4.1 of \cite{delon2020wasserstein}]
% Let $\mu_0 = \sum_{k=1}^{K_0} \pi_0^k \mu_0^k$ with $\mu_0^k = \mathcal{N} (m_0^k, \Sigma_0^k)$ and $\mu_1 = \sum_{k=1}^{K_1} \pi_1^k \delta_{m_1^k}$. Let $\title{\mu_0} =  \sum_{k=1}^{K_0} \pi_0^k \delta_{m_0^k}$. Then, \\
% $$MW_2^2 (\mu_0, \mu_1) = $$
% \end{lemma}

In Definition \ref{defwassm}, we define a Wasserstein-like distance between Gaussian Mixture Models, which uses Wasserstein-1 distance as a variation of the Proposition 4 in \cite{delon2020wasserstein}.

\begin{definition}[Wasserstein-like distance between Gaussian Mixture Models]
\label{defwassm}
Assume $P_S=\sum_{k=1}^K w_S^k P_S^k$ and  $P_T=\sum_{k=1}^K w_T^k P_T^k$ be two Gaussian mixtures. We define:
\begin{equation}
    MW_1(P_S,P_T) = \min_{w \in \Pi(w_S,w_T)} \sum_{k=1}^K \sum_{k^{\prime}=1}^K w_{k,k^{\prime}} W_1(P_S^k,P_T^k)
\end{equation}
where $\Pi(w_S,w_T)$ represents the simplex $\Delta^{K \times K}$ with marginals $w_S$ and $w_T$.
\end{definition}

% \begin{lemma}[Wasserstein-2 between Gaussian distributions]
% \label{lemmawassgaussian}
% If $P^i= \mathcal{N}(m_i,\Sigma_i)$, i $\in \{0,1\}$ are two Gaussian distributions, the 2-Wasserstein distance ($W_2$) between $P^0$ and $P^1$ has a closed-form expression, which can be written as:
% \begin{equation}
%     W_2(P^0,P^1) = \sqrt{||m_0-m_1||^2 + t}
% \end{equation}
% where t = $\text{tr}(\Sigma_0) + \text{tr}(\Sigma_1) - 2 \text{tr}(\Sigma_0^{1/2}\Sigma_1\Sigma_0^{1/2})^{1/2}$. %If we assume $P^0$ and $P^1$ are indepdent, t can be further simplified as  $\text{tr}(\Sigma_0) + \text{tr}(\Sigma_1)$
% \end{lemma}

\begin{lemma}[Extension to Wasserstein-1 - Lemma 4.1 of \cite{delon2020wasserstein}]
Explicit the distance $MW_1$ between a Gaussian mixture and a mixture of Dirac distributions.
\label{w1_proof_lemma}
Let $\mu_0 = \sum_{k=1}^{K_0} \pi_0^k \mu_0^k$ with $\mu_0^k = \mathcal{N} (m_0^k, \Sigma_0^k)$ and $\mu_1 = \sum_{k=1}^{K_1} \pi_1^k \delta_{m_1^k}$. Let $\Tilde{\mu_0} =  \sum_{k=1}^{K_0} \pi_0^k \delta_{m_0^k}$ ($\Tilde{\mu_0}$ only retains the means of $\mu_0$). Then, \\
$$MW_1(\mu_0,\mu_1)  \leq W_1 (\Tilde{\mu_0},\mu_1) + \sum_{k=1}^{K_0} \pi_0^k \sqrt{\text{tr }(\Sigma_0^k)} $$
\end{lemma}
\begin{proof}
\begin{equation}
\begin{split}
MW_1(\mu_0,\mu_1) & =\inf_{w \in \Pi(\pi_0,\pi_1)} \sum_{k,l} w_{k,l} W_1(\mu_0^k,\delta_{m_1^l}) \\
& \leq \inf_{w \in \Pi(\pi_0,\pi_1)} \sum_{k,l} w_{k,l} W_2(\mu_0^k,\delta_{m_1^l}) \\
& = \inf_{w \in \Pi(\pi_0,\pi_1)} \sum_{k,l} w_{k,l} \left[\sqrt{||m_1^l-m_0^k||^2 + \text{tr }(\Sigma_0^k)} \right] \\
%& \inf_{w \in \Pi(\pi_0,\pi_1)} \sum_{k,l} w_{k,l} \left[\sqrt{||m_1^l-m_0^k||^2 + \text{tr }(\Sigma_0^k)} \right] \\
& \leq \inf_{w \in \Pi(\pi_0,\pi_1)} \sum_{k,l} w_{k,l} ||m_1^l-m_0^k|| + \sum_{k} \pi_0^k \sqrt{\text{tr }(\Sigma_0^k)} \\
& \leq W_1 (\Tilde{\mu_0},\mu_1) + \sum_{k=1}^{K_0} \pi_0^k \sqrt{\text{tr }(\Sigma_0^k)} \\
%& \overset{\text{Def } \ref{defw1}}{=} W_1 (P_S,P_T)+ \sqrt{2\epsilon} \\
\end{split}
\end{equation}
\end{proof}

\begin{theorem}[Extension to Wasserstein-1 - Proposition 6 of \cite{delon2020wasserstein}]
\label{w1_proof}
Let $P_S$ and $P_T$ be two Gaussian mixtures. If for $\forall$ k, $k^{\prime}$, we assume $\max_{k}(\text{
trace} (\Sigma_S^k)) \leq \epsilon$ and $\max_{k^{\prime}}(\text{trace} (\Sigma_T^{k^\prime})) \leq \epsilon$.
then:
\begin{equation}
    MW_1 (P_S, P_T) \leq W_1 (P_S, P_T) + 4\sqrt{\epsilon}
\end{equation}
\end{theorem}

\begin{proof}
Here, we follow the same structure of the proof for Wassertein-2 in \cite{delon2020wasserstein}. 
Let $(P_S^n)_n$ and $(P_T^n)_n$ be two sequences of mixtures of Dirac masses respectively converging to $P_S$ and $P_T$ in $\mathcal{P}_1 (\mathbb{R}^d)$. Since $MW_1$ is a distance,
$$
\begin{aligned}
MW_1 (P_S, P_T) & \leq MW_1 (P_S^n, P_T^n) + MW_1 (P_S, P_S^n) + MW_1 (P_T, P_T^n) \\ & = W_1 (P_S^n, P_T^n) + MW_1 (P_S, P_S^n) + MW_1 (P_T, P_T^n)
\end{aligned}
$$
We can study the limits of these three terms when n $\rightarrow +\infty$

First, observe that $MW_1 (P_S^n, P_T^n) =  W_1 (P_S^n, P_T^n) \underset{n \rightarrow +\infty}{\rightarrow} W_1 (P_S,P_T)$ since $W_1$ is continuous
on $\mathcal{P}_1 (\mathbb{R}^d)$.

Second, based on Lemma \ref{w1_proof_lemma}, we have that
$$MW_1 (P_S, P_S^n) = W_1 (\Tilde{P_S}, P_S^n) + \sum_{k=1}^K w_S^k \sqrt{\text{ tr} (\Sigma_S^k)} \underset{{n \rightarrow +\infty}}{\rightarrow}  W_1 (\Tilde{P_S}, P_S) + \sum_{k=1}^K w_S^k \sqrt{\text{ tr} (\Sigma_S^k)}$$ . \\
We observe that $x\mapsto \sqrt{x}$ is a concave function, thus by Jensen's inequality, we have that
$$
\sum_{k=1}^K w_S^k \sqrt{\text{ tr} (\Sigma_S^k)} \leq \sqrt{\sum_{k=1}^K w_S^k \text{ tr} (\Sigma_S^k)}
$$
Also By Jensen's inequality, we have that, $$W_1 (\Tilde{P_S}, P_S) \leq W_2 (\Tilde{P_S},P_S)$$
And from \cite{delon2020wasserstein} we have that 
$$
 W_2 (\Tilde{P_S},P_S) \leq \sqrt{\sum_{k=1}^K w_S^k \text{ tr} (\Sigma_S^k)}
$$
Similarly for $MW_1 (P_T, P_T^n)$ the same argument holds. 
Therefore we have,
$$
    \lim_{n\rightarrow \infty} MW_1 (P_S, P_S^n) \leq 2 \sqrt{\sum_{k=1}^K w_S^k \text{ tr} (\Sigma_S^k)} 
$$
And 
$$
    \lim_{n\rightarrow \infty} MW_1 (P_T, P_T^n) \leq 2 \sqrt{\sum_{k=1}^K w_S^k \text{ tr} (\Sigma_S^k)} 
$$
We can conclude that:
\begin{equation*}
    \begin{split}
        MW_1 (P_S, P_T) & \leq \lim \inf_{n \rightarrow \infty} (W_1 (P_S^n, P_T^n) + MW_1 (P_S, P_S^n) + MW_1 (P_T, P_T^n)) \\
        & \leq W_1 (P_S, P_T) + 2\sqrt{\sum_{k=1}^K w_S^k \text{ tr} (\Sigma_S^k)} + 2\sqrt{\sum_{k=1}^K w_T^k \text{ tr} (\Sigma_T^k)} \\
        & \leq W_1 (P_S, P_T) + 4 \sqrt{\epsilon}
    \end{split}
\end{equation*}
This concludes the proof.
\end{proof}

\begin{theorem}
\label{dis_stronger_proof}
If the following assumptions hold,
\begin{itemize}
\item For $k \in \{1,\ldots,K\}$, $P^k_S$ / $P^k_T$ are Gaussian distributions with mean $m_S^k$ / $m_T^k$ and covariance $\Sigma_S^k$ / $\Sigma_T^k$.
\item The distance between the paired source-target sub-domain is smaller or equal to the distance between the non-paired source-target sub-domain, i.e., for $k \neq k^\prime$, we have $W_1(P_S^k,P_T^k) \leq W_1(P_S^k,P_T^{k^\prime})$.
\item There exists a constant $\epsilon >0$, such that $\underset{1\leq k \leq K}{\max}(\text{
trace} (\Sigma_S^k)) \leq \epsilon$ and $\underset{1\leq k \leq K}{\max}(\text{trace} (\Sigma_T^{k^\prime})) \leq \epsilon$. This $\epsilon$ is assumed to be reasonably small.
\end{itemize}
Then 
$$\sum_{k=1}^K w_T^k W_1(P_S^k,P_T^k)\leq W_1(P_S,P_T)$$
which implies Theorem \ref{theorem2} provides a stronger source-target domain discrepancy term in the generalization bound than Theorem \ref{theorem1}.
\end{theorem}
\begin{proof}
Since $w \in \Pi(w_S,w_T)$, we can write out $w_T^k$ as  $\sum_{k^{\prime}=1}^K w_{k,k^{\prime}}$, then based on assumption 2, we have:
\begin{equation*}
\begin{split}
\sum_{k=1}^K w_T^k W_1(P_S^k,P_T^k) &= \sum_{k=1}^K \sum_{k^{\prime}=1}^K w_{k,k^{\prime}} W_1(P_S^k,P_T^k) \\
& \leq \sum_{k=1}^K \sum_{k^{\prime}=1}^K w_{k,k^{\prime}} W_1(P_S^k,P_T^{k^{\prime}}) \\
\end{split}
\end{equation*}
Thus we have,
\begin{equation}
\begin{split}
\sum_{k=1}^K w_T^k W_1(P_S^k,P_T^k) 
& \leq \min_{w \in \Pi(w_S,w_T)} \sum_{k=1}^K \sum_{k^{\prime}=1}^K w_{k,k^{\prime}} W_1(P_S^k,P_T^{k^{\prime}}) \\
& = MW_1(P_S,P_T) \\
\end{split}
\end{equation}

Also we prove in Theorem \ref{w1_proof} that:
$$MW_1 (P_S, P_T) \leq W_1 (P_S, P_T) + 4\sqrt{\epsilon}$$

% \begin{equation}
% \begin{split}
% MW_1(P_S,P_T) & =\min_{w \in \Pi(w_S,w_T)} \sum_{k=1}^K \sum_{k^{\prime}=1}^K w_{k,k^{\prime}} W_1(P_S^k,P_T^{k^{\prime}}) \\
% & \leq \min_{w \in \Pi(w_S,w_T)} \sum_{k=1}^K \sum_{k^{\prime}=1}^K w_{k,k^{\prime}} W_2(P_S^k,P_T^{k^{\prime}}) \\
% & \overset{\mathrm{Lemma \ref{lemmawassgaussian}}}{\leq}  \min_{w \in \Pi(w_S,w_T)} \sum_{k=1}^K \sum_{k^{\prime}=1}^K w_{k,k^{\prime}} \left[\sqrt{||m_S^k-m_T^{k^{\prime}}||^2 + 2\epsilon} \right] \\
% & \leq \min_{w \in \Pi(w_S,w_T)} \sum_{k=1}^K \sum_{k^{\prime}=1}^K w_{k,k^{\prime}} (||m_S^k-m_T^{k^{\prime}}|| + \sqrt{2\epsilon}) \\
% %& \overset{\text{Def } \ref{defw1}}{=} W_1 (P_S,P_T)+ \sqrt{2\epsilon} \\
% \end{split}
% \end{equation}
% If one of $P_S$ and $P_T$ is Dirac masses, then $W_1(P_S,P_T) = \min_{w \in \Pi(w_S,w_T)} \sum_{k=1}^K \sum_{k^{\prime}=1}^K w_{k,k^{\prime}} ||m_S^k-m_T^{k^{\prime}}||$ \\
% To generalize the conclusion, we can extend Proposition 6 in \cite{delon2020wasserstein} to Wasserstein-1 distance, with that, we can prove $MW_1(P_S,P_T) \leq W_1(P_S,P_T) + 2\sqrt{2 \epsilon}$ (Proof in \ref{w1_proof}.)

Then we conclude our proof and show that:
\begin{equation}
\begin{split}
\sum_{k=1}^K w_T^k W_1(P_S^k,P_T^k) \leq MW_1(P_S,P_T)
\leq W_1 (P_S,P_T) + 4\sqrt{\epsilon}
\end{split}
\end{equation}

\end{proof}

\begin{theorem}[Cluster-based Regret Bound is stronger than the General Bound]\label{theorem:tigher_proof}
If the following two assumptions hold: 
\begin{enumerate}
    \item $\sum_{k=1}^{K}w^k_T\gamma_S^k(f) \leq \sum_{k=1}^{K}w^k_S\gamma_S^k(f)$.
    \item $X^k_S$ and $X^k_T$ have Gaussian distributions for $k = 1\dots K$. 
\end{enumerate}
Then we have 
\begin{equation}
    \sum_{k=1}^K w_T^k (\gamma^{k})^{\star} \leq \gamma^{\star}.
\end{equation}
Further, let 
$$
\epsilon_c(f) \doteq \sum_{k=1}^K w_T^k \gamma_S^k(f,g_S) + \sum_{k=1}^K w_T^k W_1(P_S^k,P_T^k) + \sum_{k=1}^K w_T^k (\gamma^{k})^{\star}
$$
denote the sub-domain-based generalization bound, let 
$$
\epsilon_g(f) \doteq \gamma_S(f,g_S) + W_1 (P_S, P_T) + \gamma^{\star}
$$
denote the general generalization bound without any sub-domain information. Then we have for all $f$,
\begin{equation}
    \epsilon_c(f) \leq \epsilon_g(f) + \delta_c
\end{equation}
\end{theorem}

\begin{proof}
We will proove that $\sum_{k=1}^K w_T^k (\gamma^{k})^{\star} \leq \gamma^{\star}$, where $\gamma^{\star} = \min_{f \in \mathbb{H}} \gamma_S (f,g_S) + \gamma_T (f,g_T)$, $(\gamma^{k})^{\star} = \min_{f \in \mathbb{H}} \gamma_S^k (f,g_S) + \gamma_T^k (f,g_T)$

We have:
\begin{equation}
\begin{split}
  \gamma^{\star} & = \min_{f \in \mathbb{H}} \left(\gamma_S (f,g_S) + \gamma_T (f,g_T)\right)\\
  & = \min_{f \in \mathbb{H}} \left(\sum_{k=1}^K w_S^k  \gamma_{S}^k(f,g_S) +  \sum_{k=1}^K w_T^k  \gamma_{T}^k(f,g_T)\right) \\
  & = \min_{f \in \mathbb{H}} \left( \sum_{k=1}^K w_T^k  \gamma_{S}^k(f,g_S) +  \sum_{k=1}^K w_T^k  \gamma_{T}^k(f,g_T) + \sum_{k=1}^K w_S^k  \gamma_{S}^k(f,g_S) -  w_T^k  \gamma_{S}^k(f,g_S)\right) \\
  & = \min_{f \in \mathbb{H}} \left( \sum_{k=1}^K w_T^k (\gamma_{S}^k(f,g_S) + \gamma_{T}^k(f,g_T)) + \sum_{k=1}^K (w_S^k - w_T^k) \gamma_{S}^k(f,g_S)\right) \\
  & \geq \min_{f \in \mathbb{H}} \left(\sum_{k=1}^K w_T^k (\gamma_{S}^k(f,g_S) + \gamma_{T}^k(f,g_T))\right) \\
  & \geq \sum_{k=1}^K \min_{f \in \mathbb{H}} \left(w_T^k (\gamma_{S}^k(f,g_S) + \gamma_{T}^k(f,g_T))\right) \\
  & = \sum_{k=1}^K w_T^k (\gamma^{k})^{\star}
 \end{split}
\end{equation}
where 
$(\gamma^{k})^{\star} = \min_{f \in \mathbb{H}} \gamma_S^k (f,g_S) + \gamma_T^k (f,g_T)$

The fifth step is based on the assumption that $\sum_{k=1}^{K}w^k_T\gamma_S^k(f) \leq \sum_{k=1}^{K}w^k_S\gamma_S^k(f)$.
The sixth step is based on $\min \{f(x) + g(x)\} \geq \min \{f(x)\} + \min \{g(x)\}$\\
\end{proof}

\section{Algorithm}
\label{sec:algorithm}
Our framework is outlined in pseudo-code in Algorithm \ref{alg:alg1}. 
\begin{algorithm}[h]
   \caption{Domain Adaptation via Rebalanced Sub-domain Alignment(DARSA)}
   \label{alg:alg1}
  \begin{algorithmic}
\STATE {\bfseries Input:} Source data $X_S$; Source label $y_S$, Target data $X_T$; coefficient $\lambda_Y, \lambda_D, \lambda_c, \lambda_a$; learning rate $\alpha$; \\
\STATE Pretrain feature extractor and classifier with $X_S$ and $y_S$, initialize $\theta_E^S$, $\theta_E^T$, and $\theta_Y$ with pretrained weights. Initialize $w_T^k$ and $w_T^k$ with 1/K for k = 1,2 ..., K
\REPEAT
\STATE Sample minibatch from $X_S$ and $X_T$
\STATE $\theta_Y \rightarrow \theta_Y - \alpha \nabla_{\theta_Y} (\lambda_Y \mathcal{L}_Y + \lambda_D \mathcal{L}_D + \lambda_c \mathcal{L}_{intra} + \lambda_a \mathcal{L}_{inter})$
\STATE $\theta_E^S \rightarrow\theta_E^S - \alpha \nabla_{\theta_E^S} (\lambda_Y \mathcal{L}_Y + \lambda_D \mathcal{L}_D + \lambda_c \mathcal{L}_{intra} + \lambda_a \mathcal{L}_{inter})$
\STATE $\theta_E^T \rightarrow\theta_E^T - \alpha \nabla_{\theta_E^T} (\lambda_D \mathcal{L}_D + \lambda_c \mathcal{L}_{intra} + \lambda_a \mathcal{L}_{inter})$
\UNTIL $\theta_E^S$, $\theta_E^T$, and $\theta_Y$ converge
\end{algorithmic}
\end{algorithm}

\section{Details of Experimental Setup: Digit Datasets with Shifted Class Distribution}
\label{digit_intro}
\subsection{Details of the Digit Datasets with Shifted Class Distribution}
\label{digit_intro_details_dataset}
\textbf{MNIST $\rightarrow$ MNIST-M}: For source dataset, we randomly sample 36000 images from MNIST training set with odd digits three times the even digits. For target dataset, we randomly sample 6000 images from MNIST-M constructed from MNIST testing set, with even digits three times the odd digits. To create MNIST-M dataset, we follow the procedure outlined in \cite{ganin2016domain} to blend digits from the MNIST over patches randomly extracted from color photos in the BSDS500 dataset. \cite{arbelaez2010contour}. 

\textbf{MNIST-M $\rightarrow$ MNIST}: For source dataset, we randomly sample 36000 images from MNIST-M constructed from MNIST training set, with even digits three times the odd digits. For target dataset, we randomly sample 5800 images from MNIST testing set, with odd digits three times the even digits.

\textbf{USPS $\rightarrow$ MNIST}: For source dataset, we randomly sample  3600 images from USPS training set, with even digits three times the odd digits. For target dataset, we randomly sample 5800 images from MNIST testing set, with odd digits three times the even digits.

\textbf{SVHN $\rightarrow$ MNIST}: For source dataset, we randomly sample  30000 images from SVHN training set, with even digits three times the odd digits. For target dataset, we randomly sample 5800 images from MNIST testing set, with odd digits three times the even digits.

\subsection{Additional Empirical Analysis of our Proposed Generalization
Bound}
\label{emp_tighter_more}
In our empirical analysis of our proposed generalization bound, we evaluate the bound on three additional task: 1) MNIST-M to MNIST, 2) USPS to MNIST and 3) SVHN to MNIST. All experiments are performed on our customized digit datasets with shifted class distribution (described in \ref{digit_intro_details_dataset}). As shown in Figure \ref{bound_tighter_add_figure}, our empirical results demonstrate that our theory, i.e., Theorem \ref{theorem2} provides a stronger generalization bound than than Theorem \ref{theorem1}.

\begin{figure}[h]
\begin{center}
\subfigure[Domain discrepancy]{\includegraphics[width=0.24\columnwidth]{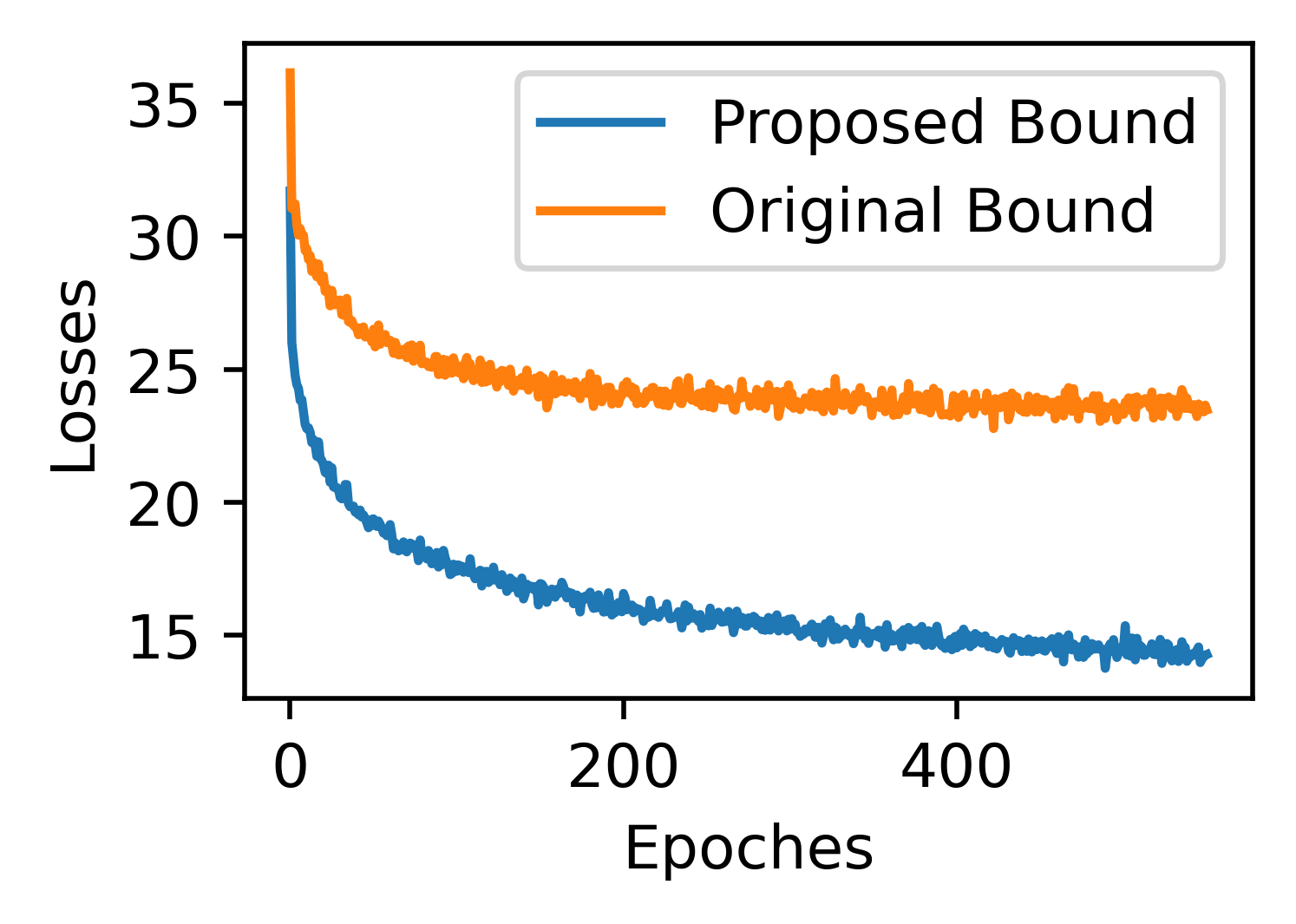}}
\subfigure[Source Classification Loss]{\includegraphics[width=0.25\columnwidth]{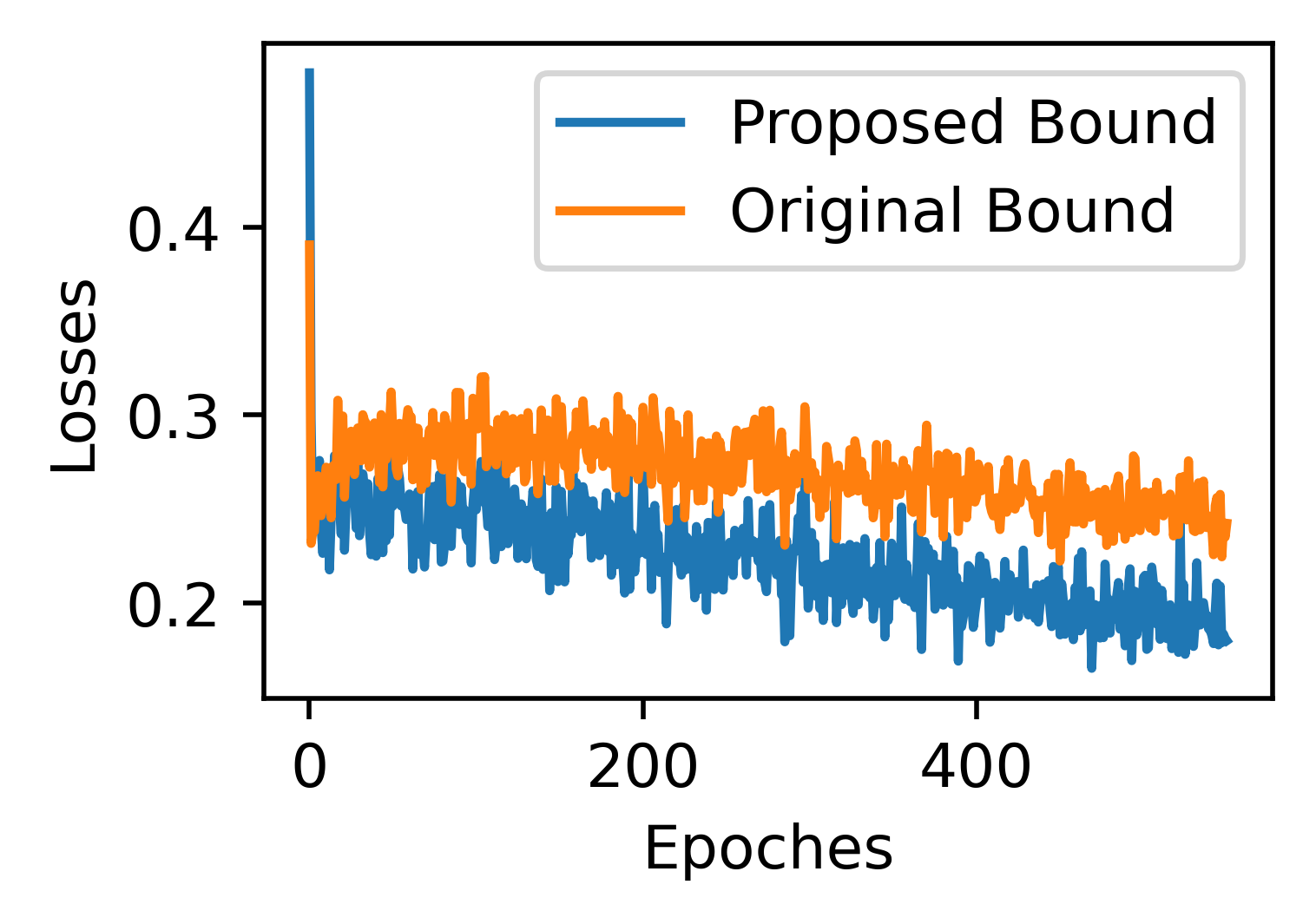}} 
\subfigure[Domain discrepancy]{\includegraphics[width=0.24\columnwidth]{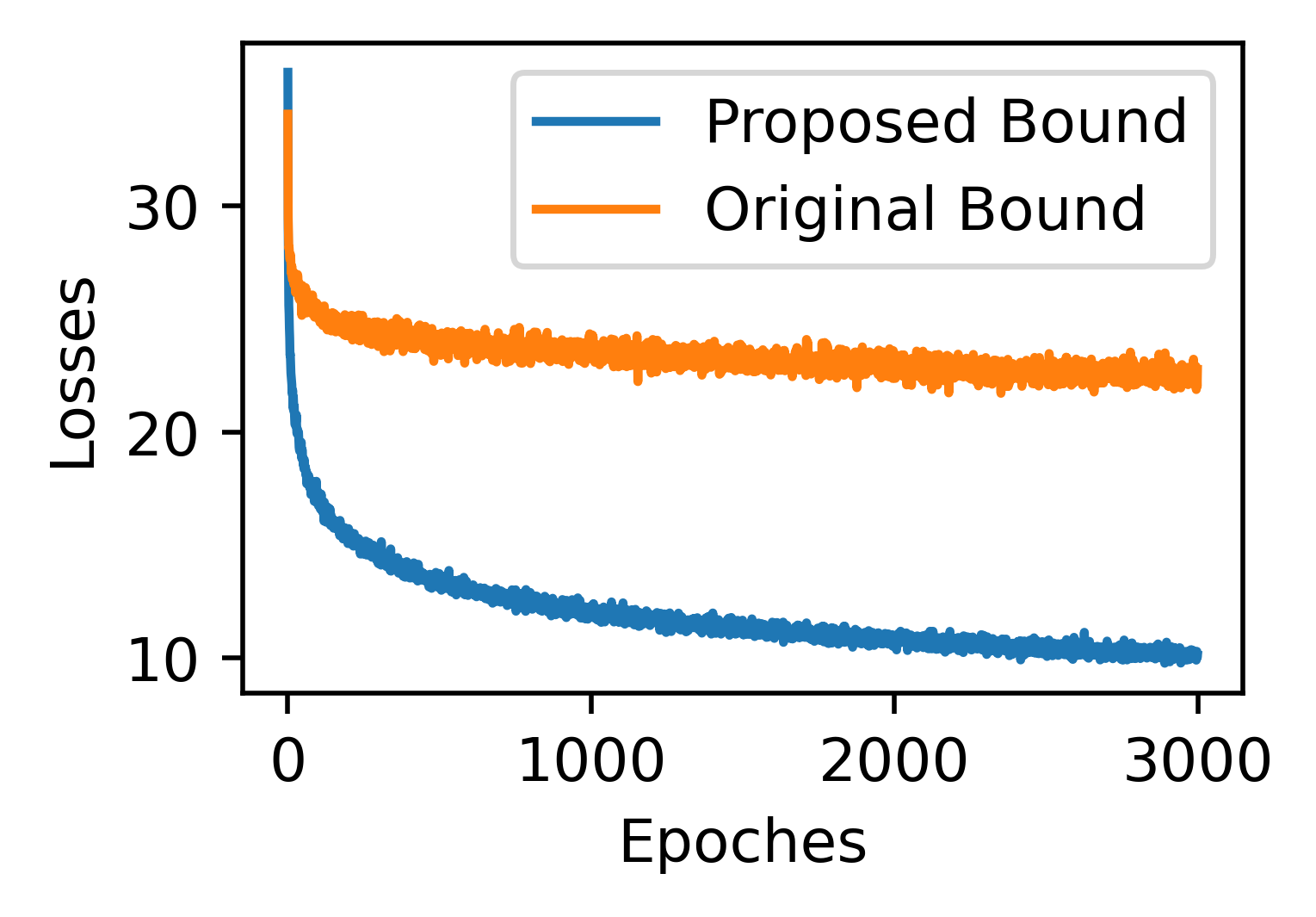}}
\subfigure[Source Classification Loss]{\includegraphics[width=0.25\columnwidth]{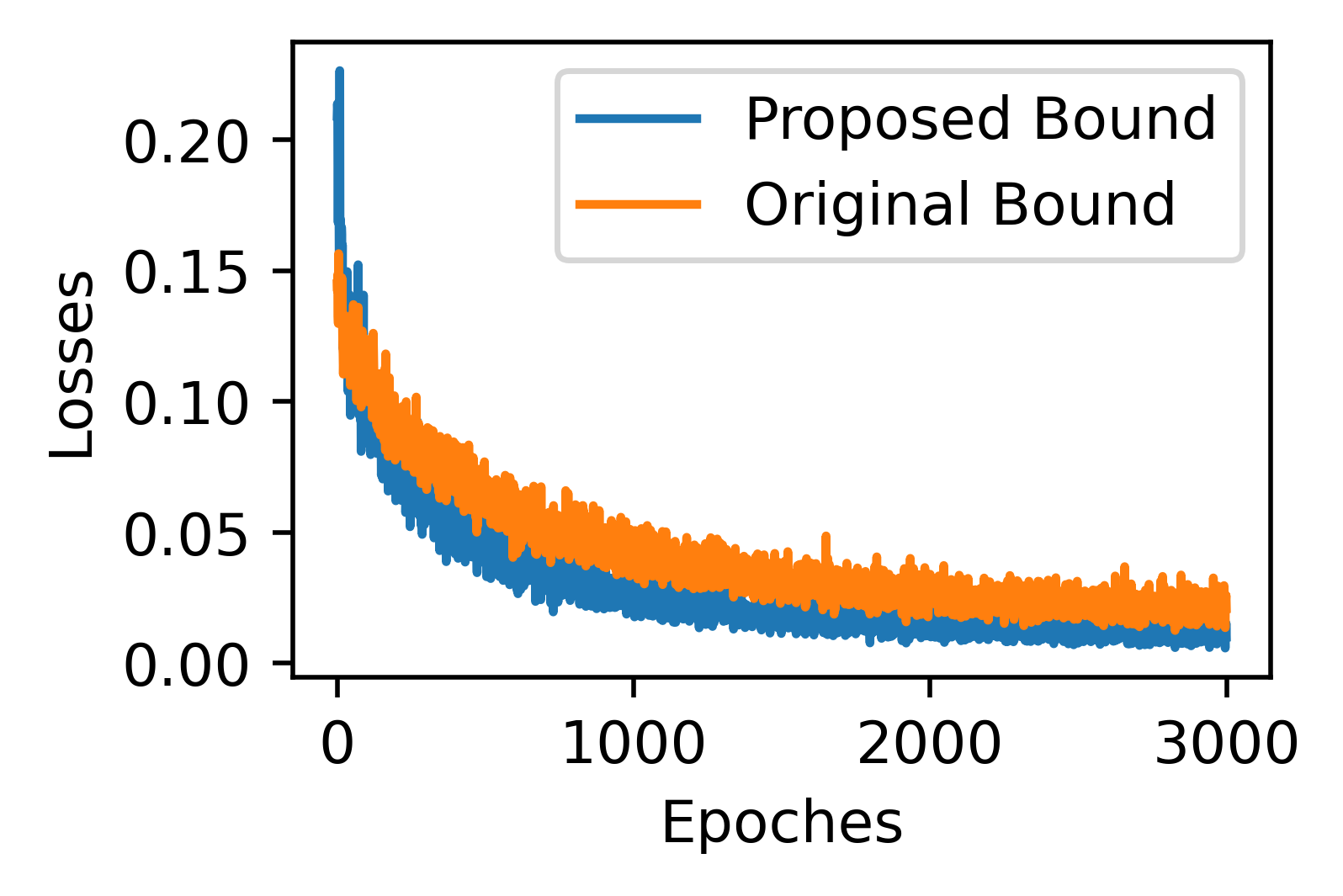}} \\
\subfigure[Domain discrepancy]{\includegraphics[width=0.24\columnwidth]{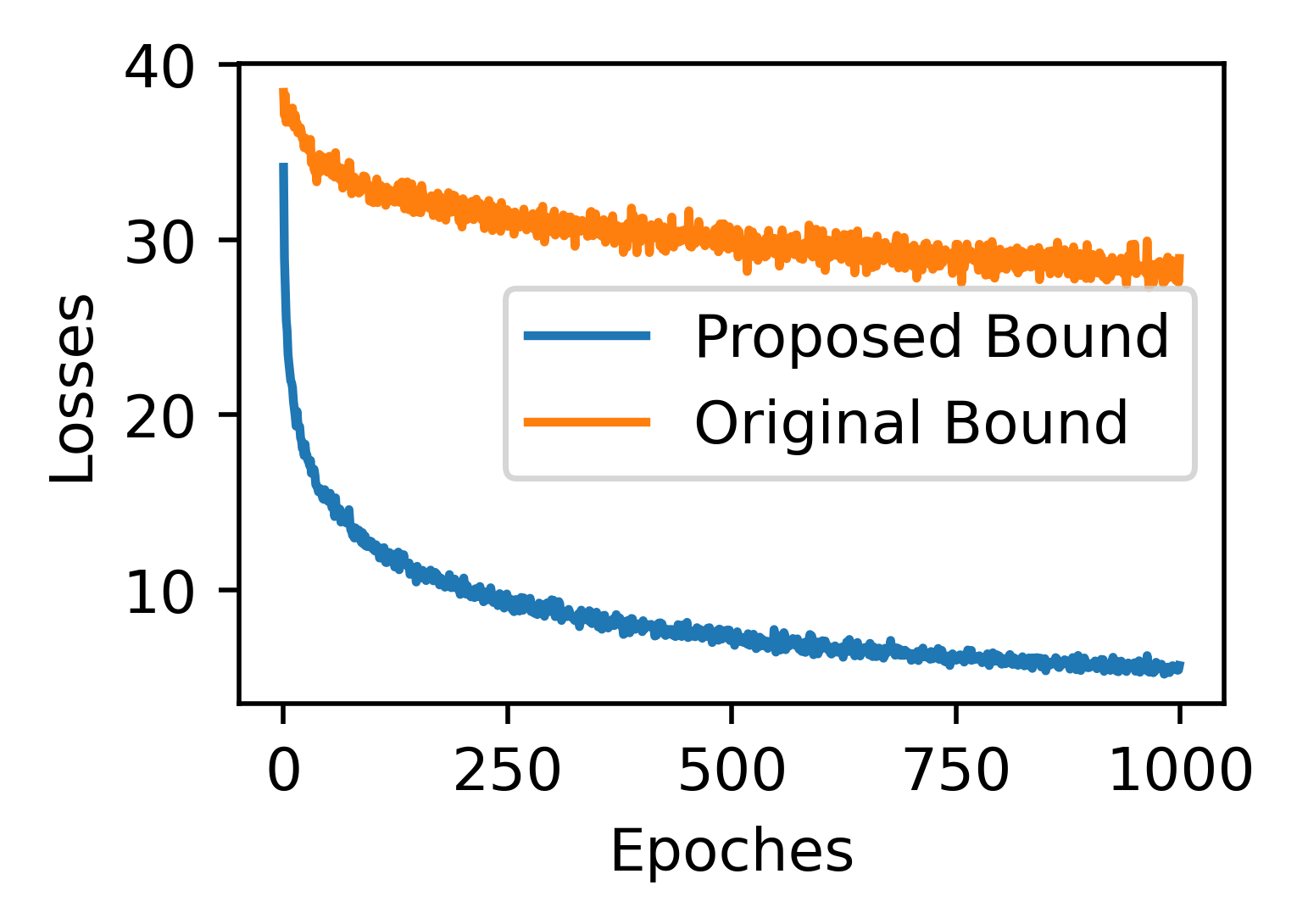}}
\subfigure[Source Classification Loss]{\includegraphics[width=0.25\columnwidth]{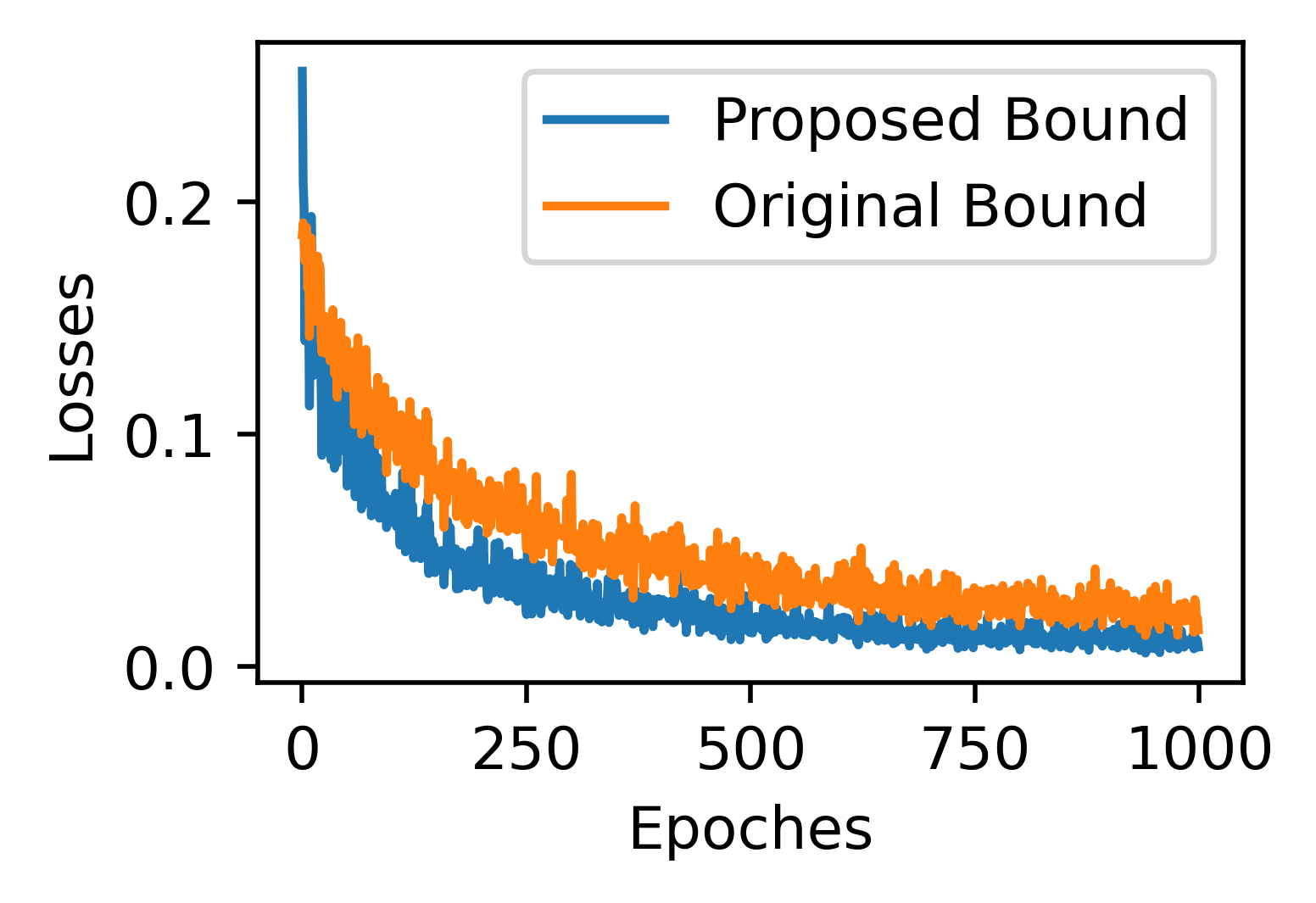}} \\
\caption{For tasks with shifted class distribution (1) we compare the domain discrepancy term ($\mathcal{L}_D$) in our proposed bound to that in Remark \ref{remark:smooth}, (2) Compare the source classification term ($\mathcal{L}_Y$) in our proposed bound to that in \ref{theorem1}. Task results shown in each subfigure: (a)(b) MNIST-M to MNIST, (c)(d) USPS to MNIST, (e)(f) SVHN to MNIST. All experiments are performed on our customized digit datasets with shifted class distribution (as described in \ref{digit_intro_details_dataset})}
\label{bound_tighter_add_figure}
\end{center}
\vskip -0.2in
\end{figure}

\subsection{Model Structures}
For feature extractor, we employ a network structure similar to LeNet-5 (\cite{lecun1998gradient}), but with minor modifications: the first convolutional layer produces 10 feature maps, the second convolutional layer produces 20 feature maps, and we use ReLU as an activation function for the hidden layer. Our feature space has 128 dimensions. For benchmarks, we utilize the network structures provided in the benchmark source code. In cases where experiments are not included in the source code, we use the same network architecture as our model to ensure fair comparisons. For classifier, we use a network structure with three fully connected layers with ReLU activation and a dropout layer with a rate of 0.5. Further details of each experiment can be found in our code. % which is available at https://anonymous.4open.science/r/DARSA/.

\subsection{Model hyperparameters}
\label{digit_hyper_details}
We use Adaptive Experimentation (Ax) platform \cite{bakshyopen2019,letham2019constrained}, an automatic tuning approaches to select hyperparameters to maximize the performance of our method. We use Bayesian optimization supported by Ax with 20 iterations to decide the hyperparameter choice. We note that most of the SOTA comparisons are
not specifically designed for shifted class distribution scenarios, and this setting caused issues in several competing. We used Ax to maximize their performance in domain shifting scenarios. Detailed model hyperparameters used for the class distribution shifting digits datasets are provided in Table \ref{tab:model_para_digit}.

\begin{table}[t!]
\centering
\begin{tabular}{c|cccc} 
\hline
~  & \begin{tabular}[c]{@{}c@{}}MNIST to \\~MNIST-M\end{tabular}  & \begin{tabular}[c]{@{}c@{}}MNIST-M to \\~MNIST\end{tabular} 
 & \begin{tabular}[c]{@{}c@{}}USPS to \\~MNIST\end{tabular}   & \begin{tabular}[c]{@{}c@{}}SVHN to \\~MNIST\end{tabular} \\ 
\hline
DARSA  & \begin{tabular}[c]{@{}c@{}} batch size = 512, \\ $\alpha= 0.01$, $\lambda_Y = 0.4$, \\ $\lambda_D = 0.35$, $\lambda_c = 1 $, \\ $\lambda_a = 0.9$, \\ m = 30 \\ SGD, momentum = 0.5~\end{tabular}  &  \begin{tabular}[c]{@{}c@{}} batch size = 512, \\ $\alpha= 0.01$, $\lambda_Y = 1$, \\ $\lambda_D = 0.5$, $\lambda_c = 1 $, \\ $\lambda_a = 1$, \\ m = 30 \\ SGD, momentum = 0.4~\end{tabular} & \begin{tabular}[c]{@{}c@{}} batch size = 256,\\ , $\alpha= 0.01$, $\lambda_Y = 1$, \\ $\lambda_D = 0.5$, $\lambda_c = 1 $, \\ $\lambda_a = 1$, \\m = 30 \\ SGD, momentum = 0.4~\end{tabular} & \begin{tabular}[c]{@{}c@{}} batch size = 256, \\ $\alpha= 0.05$, $\lambda_Y = 0.95$, \\ $\lambda_D = 0.11$, $\lambda_c = 0.3 $, \\ $\lambda_a = 0.11$,\\ m = 50 \\ SGD, momentum = 0.4~\end{tabular} \\
\hline
DANN    & \begin{tabular}[c]{@{}c@{}} batch size = 32 \\Adam, \\ learning rate = 1e-4 ~\end{tabular}         & \begin{tabular}[c]{@{}c@{}} batch size = 32 \\Adam, \\ learning rate = 1e-5 ~\end{tabular}   & \begin{tabular}[c]{@{}c@{}} batch size = 32 \\Adam, \\ learning rate = 1e-4 ~\end{tabular}  & \begin{tabular}[c]{@{}c@{}} batch size = 64 \\Adam, \\ learning rate = 1e-4 ~\end{tabular}    \\ 
\hline
WDGRL & \begin{tabular}[c]{@{}c@{}} batch size = 32 \\Adam, \\ learning rate = 1e-5, \\ $\gamma$ = 10, \\
critic training step: 1, \\
feature extractor \\
and discriminator \\ training step: 3 ~\end{tabular}  & \begin{tabular}[c]{@{}c@{}} batch size = 64 \\Adam, \\ learning rate = 1e-4, \\ $\gamma$ = 10, \\
critic training step: 5, \\
feature extractor \\ 
and discriminator \\ training step: 10  ~\end{tabular}   & \begin{tabular}[c]{@{}c@{}} batch size = 32 \\Adam, \\ learning rate = 1e-4, \\ $\gamma$ = 10, \\
critic training step: 1, \\
feature extractor\\
and discriminator \\ training step: 2 ~\end{tabular}   & \begin{tabular}[c]{@{}c@{}} batch size = 32 \\Adam, \\ learning rate = 1e-5, \\ $\gamma$ = 10, \\
critic training step: 1, \\
feature extractor \\
and discriminator \\ training step: 3 ~\end{tabular}         \\ 
\hline
DSN & \begin{tabular}[c]{@{}c@{}} batch size = 32 \\SGD, momentum = 0.8, \\
learning rate = 1e-2, \\$\alpha = 0.01$, \\
$\beta = 0.075, \gamma = 0.25$ ~\end{tabular}  & \begin{tabular}[c]{@{}c@{}} batch size = 32 \\SGD, momentum = 0.8, \\
learning rate = 0.01, \\$\alpha = 0.01$, \\
$\beta = 0.075, \gamma = 0.25$ ~\end{tabular}  & 
\begin{tabular}[c]{@{}c@{}} batch size = 32 \\SGD, momentum = 0.8, \\
learning rate = 0.01, \\$\alpha = 0.01$, \\
$\beta = 0.075, \gamma = 0.4$ ~\end{tabular} & 
\begin{tabular}[c]{@{}c@{}} batch size = 512 \\SGD, momentum = 0.5, \\
learning rate = 1e-5, \\$\alpha = 0.46$, \\
$\beta = 0.61, \gamma = 0.92$ ~\end{tabular} \\
\hline
\begin{tabular}[c]{@{}l@{}} ADDA \end{tabular} & \begin{tabular}[c]{@{}c@{}} batch size = 64 \\Adam, \\
learning rate = 1e-3, \\ 
critic training step: 1, \\
target model \\
training step: 10 ~\end{tabular}  & \begin{tabular}[c]{@{}c@{}} batch size =64 \\Adam,\\
learning rate = 1e-5, \\ 
critic training step: 1, \\
target model \\
training step: 1 ~\end{tabular} & \begin{tabular}[c]{@{}c@{}} batch size = 64 \\Adam,\\ learning rate = 1e-3, \\ 
critic training step: 1, \\
target model \\
training step: 1 ~\end{tabular} & \begin{tabular}[c]{@{}c@{}} batch size = 64 \\Adam,\\ learning rate = 1e-3, \\ 
critic training step: 3, \\
target model \\
training step: 2 ~\end{tabular} \\ 
\hline
\begin{tabular}[c]{@{}l@{}} CAT \end{tabular}     
& \begin{tabular}[c]{@{}c@{}} batch size = 512 \\SGD, \\
learning rate = $ \frac{0.01}{(1+10p)^{0.75}}$ , \\
momentum = 0.9, \\
p = 0.9, \\
m = 30 ~\end{tabular}         
& \begin{tabular}[c]{@{}c@{}} batch size = 128 \\SGD, \\
learning rate = $ \frac{0.01}{(1+10p)^{0.75}}$, \\
momentum = 0.9, \\
p = 0.9, \\
m = 30 ~\end{tabular}    
& \begin{tabular}[c]{@{}c@{}} batch size = 128 \\SGD, \\
learning rate = $ \frac{0.01}{(1+10p)^{0.75}}$, \\
momentum = 0.9, \\
p = 0.9, \\
m = 30 ~\end{tabular}        
& \begin{tabular}[c]{@{}c@{}} batch size = 256 \\SGD, \\
learning rate = $ \frac{0.01}{(1+10p)^{0.75}}$, \\
momentum = 0.9, \\
p = 0.9, \\
m = 30 ~\end{tabular}     \\ 
\hline
\begin{tabular}[c]{@{}l@{}} CDAN \end{tabular}    
& \begin{tabular}[c]{@{}c@{}} batch size = 64 \\SGD, momentum = 0.9, \\learning rate =0.01 \end{tabular}    
& \begin{tabular}[c]{@{}c@{}} batch size = 64 \\SGD, momentum = 0.9, \\learning rate =0.01 \end{tabular}    
& \begin{tabular}[c]{@{}c@{}} batch size = 64 \\SGD, momentum = 0.9, \\learning rate =0.01 \end{tabular}    
&  \begin{tabular}[c]{@{}c@{}} batch size = 64 \\SGD, momentum = 0.9, \\learning rate =0.1 \end{tabular}    \\ 
\hline
pixelDA   
& \begin{tabular}[c]{@{}c@{}} batch size = 64 \\Adam \\
learning rate =0.0002,\\
dim of the \\ noise input: 10 \end{tabular}    
& \begin{tabular}[c]{@{}c@{}} batch size = 64 \\Adam \\ learning rate =0.0002,\\
dim of the \\ noise input: 10 \end{tabular}    
& \begin{tabular}[c]{@{}c@{}} batch size = 32 \\Adam, \\ learning rate =0.0001,\\
dim of the \\ noise input: 20 \end{tabular}     
&  \begin{tabular}[c]{@{}c@{}} batch size = 32 \\Adam, \\ learning rate =0.0001,\\
dim of the \\ noise input: 20\end{tabular}   \\
\hline
DRANet  & \begin{tabular}[c]{@{}c@{}} batch size = 32 \\ Adam \end{tabular}     
& \begin{tabular}[c]{@{}c@{}} batch size = 32 \\Adam \end{tabular}      
& \begin{tabular}[c]{@{}c@{}} batch size = 32 \\Adam \end{tabular}     
& \begin{tabular}[c]{@{}c@{}} batch size = 32 \\Adam \end{tabular}    
\end{tabular}
\caption{Model hyperparameters used for digits datasets with class distribution shifting}
\label{tab:model_para_digit}
\end{table}

\section{Details of Experimental Setup: TST Dataset with class distribution shifting}
\label{neural_intro}
\subsection{Details of the TST Dataset  with Shifted Class Distribution}
The Tail Suspension Test (TST) dataset \cite{gallagher2017cross} consists of 26 mice recorded from two genetic backgrounds, Clock-$\Delta$19 and wildtype. Clock-$\Delta$19 is a genotype which has been proposed as a model of bipolar disorder while wildtype is considered as a typical or common genotype. Local field potentials (LFPs) are recorded from 11 brain regions and segmented into 1 second windows. For each window, power spectral density, coherence, and granger causality features are derived. Each mouse is placed through 3 behavioral contexts while collecting LFP recordings: home cage, open field, and tail-suspension. Mice spent 5 minutes in the home cage which is considered a baseline or low level of distress behavioral context. Mice spent 5 minutes in the open field context which is considered a moderate level of distress. Mice then spent 10 minutes in the tail suspension test which is a high distress context. Detailed model hyperparameters used for the class distribution shifting TST datasets are provided in Table \ref{tab:model_para_TST}. 

\subsection{Model Structures}
For feature extractor of the wildtype to bipolar task  we use a network structure consisting of: a fully connected layer that maps our data to a feature space of 256 dimensions, with a LeakyReLU activation function; a fully connected layer that maps the feature space to 128 dimensions, and a Softplus activation function. For the bipolar to wildtype task, we use a network structure that includes: a fully connected layer that maps our data to a feature space of 256 dimensions, with a ReLU activation function; a fully connected layer that maps the feature space to 128 dimensions, with another ReLU activation function. For the classifier, we use a network structure that includes: three fully connected layers with ReLU activation and a dropout layer with a rate of 0.5. For benchmarks, we use the same network structures as our model to ensure fair comparisons, with the exception of DSN which has two fully connected layers with ReLU activation. Additional details on each experiment can be found in our code. % which is available at https://anonymous.4open.science/r/DARSA/.

\subsection{Model hyperparameters}
Again, we use Adaptive Experimentation (Ax) platform \cite{bakshyopen2019,letham2019constrained}, an automatic tuning approaches to select hyperparameters to maximize the performance of our method. We use Bayesian optimization supported by Ax with 20 iterations to decide the hyperparameter choice. We note that most of the SOTA comparisons are not specifically designed for shifted class distribution scenarios, and this setting caused issues in several competing. We used Ax to maximize their performance in domain shifting scenarios. Detailed model hyperparameters used for the class distribution shifting digits datasets are provided in Table \ref{tab:model_para_digit}. 

\begin{table}[t]
\centering
\begin{tabular}{c|cc} 
\hline
~  & Bipolar to Wildtype  & Wildtype to Wildtype  \\ 
\hline
DARSA  & \begin{tabular}[c]{@{}c@{}} batch size = 128, \\ $\alpha$=1e-4, $\lambda_Y = 1$, \\ $\lambda_D = 0.4$, $\lambda_c = 0.1 $, \\ $\lambda_a = 0.9$, \\ m = 50 \\ SGD, momentum = 0.6~\end{tabular}  &  \begin{tabular}[c]{@{}c@{}} batch size = 128, \\ $\alpha= 0.001$, $\lambda_Y = 0.7$, \\ $\lambda_D = 0.1$, $\lambda_c = 0.1 $, \\ $\lambda_a = 1$, \\ m = 50 \\ SGD, momentum = 0.3~\end{tabular} \\
\hline
DANN    & \begin{tabular}[c]{@{}c@{}} batch size = 32 \\Adam, \\ learning rate = 1e-4 ~\end{tabular}  & \begin{tabular}[c]{@{}c@{}} batch size = 32 \\Adam, \\ learning rate = 1e-4 ~\end{tabular}    \\ 
\hline
WDGRL & \begin{tabular}[c]{@{}c@{}} batch size = 32 \\Adam, \\ learning rate = 1e-4, \\ $\gamma$ = 10, \\
critic training step: 1, \\
feature extractor\\
and discriminator \\ training step: 2 ~\end{tabular}   & \begin{tabular}[c]{@{}c@{}} batch size = 32 \\Adam, \\ learning rate = 1e-5, \\ $\gamma$ = 10, \\
critic training step: 1, \\
feature extractor \\
and discriminator \\ training step: 3 ~\end{tabular}         \\ 
\hline
DSN  & \begin{tabular}[c]{@{}c@{}} batch size = 64 \\SGD, momentum = 0.5, \\
learning rate = 0.1, \\$\alpha = 1$, \\
$\beta = 1, \gamma = 1$ ~\end{tabular}  & \begin{tabular}[c]{@{}c@{}} batch size = 32 \\SGD, momentum = 0.5, \\
learning rate = 0.1, \\$\alpha = 1$, \\
$\beta = 1, \gamma = 1$ ~\end{tabular} \\
\hline
\begin{tabular}[c]{@{}l@{}} CAT \end{tabular}     
& \begin{tabular}[c]{@{}c@{}} batch size = 128 \\SGD, \\
learning rate = $ \frac{0.01}{(1+10p)^{0.75}}$, \\
momentum = 0.9, \\
p = 0.9, \\
m = 3 ~\end{tabular}         
& \begin{tabular}[c]{@{}c@{}} batch size = 128 \\SGD, \\
learning rate = $ \frac{0.01}{(1+10p)^{0.75}}$, \\
momentum = 0.9, \\
p = 0.9, \\
m = 3 ~\end{tabular}     \\
\hline 
\begin{tabular}[c]{@{}l@{}} CDAN \end{tabular}    
& \begin{tabular}[c]{@{}c@{}} batch size = 64 \\SGD, momentum = 0.9, \\learning rate = 0.1 \end{tabular}  
& \begin{tabular}[c]{@{}c@{}} batch size = 64 \\SGD, momentum = 0.9, \\learning rate = 0.1 \end{tabular}  \\ 
\hline
\end{tabular}
\caption{Model hyperparameters used for the class distribution shifting TST datasets}
\label{tab:model_para_TST}
\end{table}

\section{Accessibility of the Datasets}
\renewcommand\thefigure{\thesection\arabic{figure}}
\setcounter{figure}{0}
\label{appx:E}
The MNIST, BSDS500, USPS, and SVHN datasets are publicly available with an open-access license. The Tail Suspension Test (TST) dataset \cite{gallagher2017cross} is available by request from the authors of \cite{gallagher2017cross}.

\end{document}